\documentclass[twoside,11pt]{article}

\usepackage{color}
\usepackage{jmlr2e}
\usepackage{graphics}
\usepackage{epsfig}
\usepackage[ruled]{algorithm2e}
\usepackage{booktabs}
\usepackage{url}
\usepackage{psfrag}
\usepackage{relsize}
\usepackage{amsmath,amsfonts,mathrsfs,mathtools,amssymb}
\usepackage[colorlinks,linkcolor=blue,anchorcolor=green,citecolor=green]{hyperref}
\usepackage{tikz,pgfplots}
\usepackage{tkz-tab}
\usepackage[format=hang,justification=justified,singlelinecheck=false,font=footnotesize]{caption}
\usepackage[font=footnotesize]{subcaption}

\usepackage{enumerate}

\usepackage[T1]{fontenc}
\usepackage[utf8]{inputenc}
\usepackage[font=small,labelfont=bf,tableposition=top]{caption}
\usepackage{booktabs}
\usepackage{threeparttable}

\usepackage{multirow}

\newcommand{\eins}{\boldsymbol{1}}

\DeclareSymbolFont{wideparensymbol}{OMX}{yhex}{m}{n}
\DeclareMathAccent{\wideparen}{\mathord}{wideparensymbol}{"F3}

\jmlrheading{}{}{~}{~}{~}{2018 Hanyuan Hang and Hongwei Wen}

\usepackage{tikz,pgfplots}
  \pgfplotsset{compat=newest}
  \pgfplotsset{plot coordinates/math parser=false,trim axis left}
     \newlength\figureheight
     \newlength\figurewidth

\begin{document}

\title{Best-scored Random Forest Density Estimation}

\author{\name Hanyuan Hang \email hans2017@ruc.edu.cn \\
	\name Hongwei Wen \email hongwei.wen@ruc.edu.cn \\
	\addr Institute of Statistics and Big Data \\  
	Renmin University of China \\
	100872 Beijing, China}


\maketitle

\begin{abstract}This paper presents a brand new nonparametric density estimation strategy named the best-scored random forest density estimation whose effectiveness is supported by both solid theoretical analysis and significant experimental performance. The terminology best-scored stands for selecting one density tree with the best estimation performance out of a certain number of purely random density tree candidates and we then name the selected one the best-scored random density tree. In this manner, the ensemble of these selected trees that is the best-scored random density forest can achieve even better estimation results than simply integrating trees without selection. From the theoretical perspective, by decomposing the error term into two, we are able to carry out the following analysis: First of all, we establish the consistency of the best-scored random density trees under $L_1$-norm. Secondly, we provide the convergence rates of them under $L_1$-norm concerning with three different tail assumptions, respectively. Thirdly, the convergence rates under $L_{\infty}$-norm is presented. Last but not least, we also achieve the above convergence rates analysis for the best-scored random density forest. When conducting comparative experiments with other state-of-the-art density estimation approaches on both synthetic and real data sets, it turns out that our algorithm has not only significant advantages in terms of estimation accuracy over other methods, but also stronger resistance to the curse of dimensionality. 
\end{abstract}


\section{Introduction}
Owing to the rapid development of computation and the consequent emergence of various types of data, effective tools to deal with data analysis are in great demand. Among those tools, density estimation which aims at estimating the underlying density of an unknown distribution through observations drawn independently from that distribution has been attached paramount importance in many fields of science and technology 
\cite{fraley2002modelbased}. This broad attention is a direct result of the fact that the density estimation does not learn for its own sake, but rather facilitate solving some higher level tasks, such as assessing the multimodality, skewness, or any other structure in the distribution of the data \cite{scott1992multivariate,silverman1986density}, summarizing the Bayesian posteriors, classification and discriminant analysis \cite{simonoff1996smoothing}, and being proved useful in Monte Carlo computational methods like bootstrap and particle filter \cite{doucet2001an}. Other applications, especially in the computer vision society, include image detection \cite{ma2015small,liu2016highway,wang2018manifoldbased}, gesture recognition \cite{chang2016nonparametric}, image reconstruction \cite{ihsani2016a}, deformable 3D shape matching \cite{vestner2017product}, image defogging \cite{jiang2017fog}, hyperspectral unmixing \cite{zhou2018a}, just to name a few.

Decades have witnessed vast literature on finding different appropriate methods to solve density estimation problems and nonparametric density estimations have become the focus of attention since weaker assumptions are applied to the underlying probability distribution \cite{hwang1994nonparametric,hardle2012nonparametric}. Histogram density estimation, being a simple and convenient estimation method, has been extensively studied as the most basic form of density estimation \cite{freedman1981on,lugosi1996consistency}. Although consistency
\cite{glick1973samplebased,gordon1978asymptotically,gordon1980consistent} and a strong universal consistency \cite{devroye1983distributionfree} of the histograms are established, they are suboptimal for not being smooth. Moreover, the non-smoothness of histogram and therefore insufficient accuracy brings great obstacles to practical applications. Taking the smoothness into consideration, the machine learning society turns to another popular strategy called the kernel density estimation (KDE), which is also termed as \emph{Parzen-Rosenblatt estimation} \cite{parzen1962on,rosenblatt1956remarks}. This method gains its prevalence when dealing with cases where the density is assumed to be smooth and the optimal convergence rates can be achieved with kernel and bandwidth chosen appropriately \cite{hang2018kernel}. However, these optimal rates depend on the order of smoothness of the density function on the entire input space while the actual cases may be that the smoothness of density function varies from areas to areas. In other words, KDE lacks local adaptivity, and this often leads to a large sensitivity to outliers, the presence of spurious bumps, and a tendency to flatten the peaks and valleys of the density \cite{terrell1992variable,botev2010kernel}. Nevertheless, this method undergoes a high computational complexity since the computation time grows linearly with the number of samples increasing. Other density estimation strategies published so far include estimators based on wavelet \cite{doukhan1990deviation,kerkyacharian1992density}, mixtures of models \cite{roeder1997practical,ghosal2001convergence,escobar1995bayesian}, just to name a few. It is worth noting that the above mentioned methods can barely escape from the curse of dimensionality for their unsatisfying performance for moderate to large dimension. To the best of our knowledge, it is a challenge for an algorithm to have the theoretical availability for both local and global analysis, the experimental advantages of achieving efficient and accurate prediction results on real data, and stronger resistance to the curse of dimensionality compared to the existing common algorithms.

Committed to conquering the challenge, we propose a random-forest-based density estimation method named the \emph{best-scored random forest density estimation}. By taking full advantage of the purely random splitting criterion and the ensemble nature of a forest consisting of purely random trees, we are able to construct an algorithm not only achieving fast convergence rates, but also a desirable asymptotic smoothness beneficial for prediction accuracy. Moreover, since the local and global analysis of the random forests are in essence the same, so is our algorithm. The algorithm starts with partitioning the feature space into non-overlapping cells following the purely random splitting criterion where at each step, the to-be-split cell and its corresponding cut point are chosen uniformly at random. Then, the inherent randomness within the partitions allows us to build different random density trees and pick out the one with the best experimental performance as a best-scored tree in the forest. We name this selection mechanism the best-scored method. Last but not least, by integrating trees generated by the above procedure, we obtain a density forest with satisfying asymptotic smoothness.

The contributions of this paper come from both the theoretical and experimental aspects: 
\emph{(i)} 
Considering from a theoretical perspective, our best-scored random forest density estimation shows its preponderance in achieving fast convergence rates under some mild conditions. Different from the commonly utilized $L_2$-distance for measuring the difference between the nonparametric density estimator and the underlying density function, we regard $L_1$-distance as a more reasonable choice concerning with its invariance under monotone transformations, being well-defined as a metric on the density functions space and better visualization of the closeness to the ground-truth density function than $L_2$-distance. Besides, we also carry out analysis under the $L_{\infty}$-distance for its ability to measure the worst-case goodness-of-fit of the estimator. Based on these two types of distance, we manage to establish consistency and fast convergence rates for the best-scored random density trees and forest. In the analysis, the error term is decomposed into data-free and data-dependent error terms which are handled by employing techniques from the approximation theory and empirical process theory, respectively. Our theoretical advantages are essentially twofold: First, the underlying density function is only assumed to be $\alpha$-H\"{o}lder continuous which is a weak and natural assumption for nonparametric density estimation in the literature. Second, fast convergence rates are established under certain common tail assumptions on the distribution which are rigorously calculated step by step according to our purely random splitting mechanism. 
\emph{(ii)} 
Experimental improvements of the algorithm architecture are made for better numerical performance and their effectiveness is later verified by both synthetic and real data analysis. First of all, we adopt an adaptive random partition method instead of the original random splitting criterion where the cell selection process is data-driven. To be concrete, at each step, we pick up a certain number of sample points from the entire training set uniformly at random and then choose the cell which most of these samples fall in as the to-be-split one. In this manner, sample-dense areas are more likely to be split more whereas sample-sparse areas are possible to be split less, which not only increases the effective number of splits, but also helps to obtain cells with sample sizes evenly distributed. Secondly, concerning with the fact that the partitions theoretically studied are axis-parallel and may not be that accurate since it has to approximate the correct model with staircase-like structure when the underlying concept is a polygonal space partitioning in practice. Therefore, we propose a best-scored random forest density estimation induced by the oblique purely random partitions and it does improve the prediction accuracy. Thirdly, when conducting real data analysis, our algorithm is predominant in accuracy for it has much more free parameters tunable, trains faster than other classical machine learning methods when the sample volume is large, can be even speed up for it inherits the parallelism of random forests and significantly more resistant to the curse of dimensionality than any other methods in comparison. As a result, the noteworthy advantages of experimental accuracy and training time further demonstrate the effectiveness and efficiency of our algorithm.

This paper is organized as follows. In Section \ref{sec::methodology}, we lay out some required fundamental notations and definitions concerned with the best-scored random density forest. Main results on the consistency and convergence rates under the $L_1$-norm and the $L_{\infty}$-norm of the estimators are provided in Section \ref{sec::consistency_convergence}. Some comments and discussions related to the main results will be also presented in this section. Section \ref{error_analysis} is devoted to the main analysis on bounding the error terms. Numerical experiments of comparisons between different density estimation methods based on both synthetic and real data sets are provided in Section \ref{sec::numerical_experiments}. For the sake of clarity, we place all the proofs of Section \ref{sec::consistency_convergence} and Section \ref{error_analysis} in Section \ref{sec::proofs}. We close this paper in Section \ref{sec::conclusion} with several concluding remarks.

\section{Methodology} \label{sec::methodology}

We dedicate this section to the methodology of our best-scored random forest density estimation. To this end, we begin by introducing some notations that will be used throughout. Then, we give explicit description of the purely random partitions that our density trees and thus forest are based on. The architecture of our best-scored random density trees and then forest are presented in Sections \ref{sec::subsec::BRDT} and \ref{sec::subsec::BRDF}, respectively.

\subsection{Notations}

Let $\mathcal{X} \subset \mathbb{R}^d$ be a subset, $\mu := \lambda^d$ be the Lebesgue measure with $\mu(\mathcal{X}) > 0$, and $\mathrm{P}$ be a probability measure with support $\mathcal{X}$ which is absolute continuous with respect to $\mu$ with density $f$. We denote $B_r$ as the centered hypercube of $\mathbb{R}^d$ with side length $2 r$, that is
\begin{align*}
B_r := \{ x = (x_1, \ldots, x_d) \in \mathbb{R}^d : x_i \in [-r, r], i = 1, \ldots, d \},
\end{align*}
and write $B_r^c := \mathbb{R}^d \setminus [-r, r]^d$ for the complement of $B_r$. Throughout this paper, we use the notation $a_n \lesssim b_n$ to denote that there exists a positive constant $c$ such that $a_n \leq c b_n$, for all $n \in \mathbb{N}$.

\subsection{Purely Random Partitions} \label{sec::subsec::PRP}

In this subsection, we introduce the purely random partition which is the foundation of establishing our best-scored random density trees and then forest. This partition follows the idea put forward by \cite{bremain2000some} of the construction of purely random forest. To give a better understanding of one possible general building procedure of the random partition, a random vector $Q_i := (L_i, R_i, S_i)$ is set up to describe the splitting mechanism at the $i$th step of the partition. For definiteness, let $L_i$ in the triplet denote the to-be-split cell at the $i$th step chosen uniformly at random from the candidates which are defined to be all the cells presented in the $(i-1)$th step. In this way, the cell choosing procedure follows a recursive manner. The second random variable $R_i \in \{ 1, \ldots, d \}$ in the triplet denotes the dimension chosen to be split from for cell $L_i$ and $\{ R_i, i \in \mathbb{N}_+ \}$ are independent and identically multinomial distributed with each dimension having equal probability to be chosen. The random variable $S_i$ serves as a proportional factor representing the ratio between the length of the newly generated cell in the $R_i$th dimension after the $i$th split and the length of the being-cut cell $L_i$ in the $R_i$th dimension. That is to say, the length of the newly generated cell in the $R_i$th dimension is the product of the length of $L_i$ in the $R_i$th dimension and the proportional factor $S_i$. Note that $\{ S_i, i \in \mathbb{N}_+ \}$ are independent and identically distributed drawn from $\mathcal{U}[0, 1]$.

The above mentioned statements mathematically formulate the splitting process of the purely random tree. However, one simple example may provide a clearer understanding of the whole procedure. To be specific, we assume that the partition is carried out on $B_r := [-r, r]^d$, $r \geq 1$. First of all, we randomly select one dimension out of $d$ candidates and uniformly split at random from that dimension. The split is a $(d - 1)$-dimensional hyperplane parallel to the axis so that $B_r$ is split into two cells which are $A_{1,1}$ and $A_{1,2}$, respectively. Then, a cell is chosen uniformly at random, say $A_{1,1}$, and we conduct random split on it with dimension and cut point chosen randomly, which leads to a partition of $B_r$ consisting of $A_{2,1}, A_{2,2}, A_{1,2}$. Next, we randomly pick one cell from all three cells formed in the last step, say $A_{2,2}$, and the split is conducted on it as before, which leads to a partition consisting of $A_{2,1}, A_{3,1}, A_{3,2}, A_{1,2}$. The building process continues in this manner until the number of splits $p$ meets our satisfaction. Moreover, the above procedure leads to a so-called partition variable $Z := (Q_1, \ldots, Q_p, \ldots)$ taking value in space $\mathcal{Z}$. We denote by $\mathrm{P}_Z$ the probability measure of $Z$.

It is worth pointing out that any specific partition variable $Z \in \mathcal{Z}$ can be treated as a splitting criterion. The collection of non-overlapping cells formed by following $Z$ for $p$ splits on $B_r$ is denoted by $\mathcal{A}_{(Q_1, \ldots, Q_p)}$ which can be further abbreviated as $\mathcal{A}_{Z,p}$, and we define $\mathcal{A}_{Z,0} := B_r$. We also note that if we focus on certain sample point $x \in B_r$, then the corresponding cell where that point falls is denoted by $A_{Z,p}(x)$.

\subsection{Best-scored Random Density Trees} \label{sec::subsec::BRDT}

In this subsection, we formulate the best-scored random density tree (BRDT) based on the above mentioned random partitions $\mathcal{A}_{Z,p}$. We first introduce how to build a density tree based on purely random partition, then incorporate the best-scored method into the construction of trees, which leads to our best-scored random density trees.

\subsubsection{Purely Random Density Tree}

In order to characterize the purely random density tree estimators, we propose the following definition formalizing the general form of random partition.

\begin{definition}[Random Partition]
	For a fixed $r \geq 1$, let $Z$ be a random splitting criterion of $B_r := [-r, r]^d$. The collection of non-overlapping sets $\mathcal{A}_{Z,p} := \{ A_j, j = 0, \ldots, p \}$ derived by partitioning $B_r$ following $Z$ for $p$ splits is called a $p$-split random partition. And each element in $\mathcal{A}_{Z,p}$ is called a cell of the random partition.
\end{definition}

Now, we introduce the random density tree with respect to certain probability measure. There is no loss of generality in assuming that for all $A \in \mathcal{A}_{Z,p}$, the Lebesgue measure $\mu(A) > 0$, since the density estimation at $x$ is set to be $0$ if $\mu(A(x)) = 0$. From now on, this assumption will hold without repetition.

\begin{definition}[Random Density Tree of a Measure]
	Let $Q$ be a probability measure on $\mathbb{R}^d$. For a fixed $r \geq 1$, let $\mathcal{A}_{Z,p} := \{ A_j, j = 0, \ldots, p \}$ be a $p$-split random partition of $B_r$. Then, the function $f_{Q,Z,p} : \mathbb{R}^d \to [0, \infty)$ defined by
	\begin{align} \label{RandomDensityTreeQ}
	f_{Q,Z,p}(x) := \sum_{j=0}^p \frac{Q(A_j) \eins_{A_j}(x)}{\mu(A_j)} + \frac{Q(B_r^c) \eins_{B_r^c}(x)}{\mu(B_r^c)}
	\end{align}
	is called a random density tree of $Q$.
\end{definition}

In the following, we write $f_{Q,Z}$ instead of $f_{Q,Z,p}$ for abbreviation. Here, we demonstrate that $f_{Q,Z}$ defines the density of a probability measure on $\mathbb{R}^d$ for $f_{Q,Z}$ is measurable and
\begin{align}
\int_{\mathbb{R}^d} f_{Q,Z} \, d\mu
& = \int_{\mathbb{R}^d} \sum_{j=0}^p \frac{Q(A_j) \eins_{A_j}(x)}{\mu(A_j)} + \frac{Q(B_r^c) \eins_{B_r^c}(x)}{\mu(B_r^c)} \, d\mu(x)
\nonumber\\
& = \sum_{j=0}^p \int_{\mathbb{R}^d} \frac{Q(A_j) \eins_{A_j}(x)}{\mu(A_j)} \, d\mu(x) 
+ \int_{\mathbb{R}^d} \frac{Q(B_r^c) \eins_{B_r^c}(x)}{\mu(B_r^c)} \, d\mu(x)
\nonumber\\
& = \sum_{j=0}^p \frac{Q(A_j) \mu(A_j)}{\mu(A_j)} + \frac{Q(B_r^c) \mu(B_r^c)}{\mu(B_r^c)} = 1. 
\label{DensityProperty}
\end{align}
Moreover, for $x \in A_j$, $j \in \{ 0, 1, \ldots, p \}$, we have
\begin{align*}
f_{Q,Z}(x) = \frac{Q(A_j)}{\mu(A_j)},
\end{align*}
which also holds for $x \in B_r^c$.

Recalling that $\mathrm{P}$ is a probability measure on $\mathbb{R}^d$ with the corresponding density function $f$, by taking $Q = \mathrm{P}$ with $d\mathrm{P} = f \, d\mu$, then for $x \in A_j$, we have
\begin{align} \label{PDensity1}
f_{\mathrm{P},Z}(x) 
= \frac{\mathrm{P}(A_j)}{\mu(A_j)} 
= \frac{1}{\mu(A_j)} \int_{A_j} f(x') \, d\mu(x'). 
\end{align}

In other words, $f_{\mathrm{P},Z}$ in $A_j$ is the average density on $A_j$. Furthermore, for $x \in B_r$, there exists exactly a number $j \in \{ 0, 1, \ldots, p \}$ such that $x \in A_j$. In the following, we write $A(x) := A_j$. Then, for $x \in B_r$ with $\mu(A(x)) > 0$,
\begin{align} \label{PDensity2}
f_{\mathrm{P},Z}(x)
= \frac{\mathrm{P}(A(x))}{\mu(A(x))}
= \frac{1}{\mu(A(x))} \int_{A(x)} f(x') \, d\mu(x'). 
\end{align}
Specifically, when $Q$ is the empirical measure $\mathrm{D}_n = \frac{1}{n} \sum_{i=1}^n \delta_{x_i}$, then $\mathrm{D}_n(A)$ is the expectation of $\eins_A$ with respect to $\mathrm{D}_n$, which is
\begin{align*}
\mathrm{D}_n(A)
= \mathbb{E}_{\mathrm{D}_n} \eins_A
= \frac{1}{n} \sum_{i=1}^n \delta_{x_i}(A)
= \frac{1}{n} \sum_{i=1}^n \eins_A(x_i).
\end{align*}
For $x \in A_j$, the random density tree in this study can be expressed as
\begin{align} \label{RandomDensityTree}
f_{\mathrm{D}_n,Z}(x)
= \frac{\mathrm{D}_n(A_j)}{\mu(A_j)}
= \frac{1}{n \mu(A_j)} \sum_{i=1}^n \eins_{A_j}(x_i)
\end{align}
where $A_j$ can also be written as $A(x)$. The summation on the right-hand side of \eqref{RandomDensityTree} counts the number of observations falling in $A_j$. From now on, for notational simplicity, we will suppress the subscript $n$ of $\mathrm{D}_n$ and denote $\mathrm{D} := \mathrm{D}_n$, e.g., $f_{\mathrm{D},Z} := f_{\mathrm{D}_n,Z}$. The map from the training data to $f_{\mathrm{D},Z}$ is called a random density tree rule with random partition $\mathcal{A}_{Z,p}$.

\subsubsection{The Best-scored Method}

We should attach great importance to the fact that the prediction performances of density trees induced by purely random partitions might not be that satisfying since the partitions completely make no use of the sample information. Therefore, the prediction results of their ensemble forest may not be accurate enough. Committed to improving prediction accuracy, we provide a selection process for the partitioning of each tree. Concretely speaking, the partition chosen for tree construction is the one with the best density prediction performance in terms of the Average Negative Log-Likelihood (\emph{ANLL}) (to be mentioned later in Section \ref{sec::subsec::experisetup}) from $k$ partition candidates. This process is named as the \emph{best-scored} method and the resulting trees are called the \emph{best-scored random density trees}.

\subsection{Best-scored Random Density Forest} \label{sec::subsec::BRDF}

In this subsection, we formulate the best-scored random density forest. Ensembles consisting of a combination of different estimators have been highly recognized as an effective technique for the significant performance improvements over single estimator in the literature, which inspires us to apply them to our best-scored density trees. In our cases, we first train several best-scored density trees basing on different random partitions, separately; once this is accomplished, the outputs of the individual estimators are combined to give the ensemble output for new data points. Here, we use the simplest possible combination mechanism by taking uniform weighted average.

Let $f_{\mathrm{D},Z_t}$, $1 \leq t \leq m$ be the $m$ best-scored random density tree estimators generated by the splitting criteria $Z_1, \ldots, Z_m$ respectively, which is defined by
\begin{align*}
f_{\mathrm{D},Z_t}(x)
:= \sum_{j=0}^p \frac{\mathrm{D}(A_{tj}) \eins_{A_{tj}}(x)}{\mu(A_{tj})}
+ \frac{\mathrm{D}(B_r^c) \eins_{B_r^c}(x)}{\mu(B_r^c)},
\end{align*}
where $\mathcal{A}_{Z_t} := \{ A_{tj}, j = 0, \ldots, p \}$ is a random partition of $B_r$. Therefore, with $Z_{\mathrm{E}} := \{ Z_1, \ldots, Z_m \}$, the best-scored random density forest can be presented as
\begin{align} \label{RandomDensityForest}
f_{\mathrm{D},Z_{\mathrm{E}}}(x)
:= \frac{1}{m} \sum_{t=1}^m f_{\mathrm{D},Z_t}(x).
\end{align}

\section{Main Results and Statements} \label{sec::consistency_convergence}

In this section, we present main results on the consistency and convergence rates of our density estimators. To be precise, consistency and convergence rates of the best-scored random density trees under $L_1$-norm are given in Sections \ref{sec::subsec::universal_consis} and \ref{subsec::l1_convergence}, respectively. Convergence rates of the best-scored random density trees under $L_{\infty}$-norm are presented in Section \ref{subsec::l_infty_convergence}. Based on the results of those base density estimators, the convergence rates of the best-scored random density forest under $L_1$-norm and $L_{\infty}$-norm are established in Section \ref{subsec::l_infty_convergence_forest}. Finally, comments and discussions concerned with the established main results are given in Section \ref{CandD}.

\subsection{Results on Consistency} \label{sec::subsec::universal_consis}

We establish results on the consistency property of the best-scored random density tree estimator $f_{\mathrm{D},Z}$ in the sense of $L_1$-norm. To clarify, an estimator $f_{\mathrm{D},Z}$ is said to be \emph{consistent} in the sense of $L_1$-norm if $f_{\mathrm{D},Z}$ converges to $f$ under $L_1$-norm $\mathrm{P}^n \otimes \mathrm{P}_Z$-almost surely.

\begin{theorem} \label{L1Convergence} 
	For $n \geq 1$, let $\mathcal{A}_{Z,p_n}$ be a random partition with number of splits $p_n$. If
	\begin{align*}
	p_n \to \infty 
	\qquad
	\text{ and } 
	\qquad 
	\frac{p_n \log n}{n} \to 0, 
	\qquad \qquad
	\text{ as } 
	n \to \infty,
	\end{align*}
	then the best-scored random density tree estimator $f_{\mathrm{D},Z}$ is consistent in the sense of $L_1$-norm.
\end{theorem}

\subsection{Results on Convergence Rates under $L_1$-Norm} \label{subsec::l1_convergence}

In this subsection, we establish the convergence rates of the best-scored random density tree estimators under $L_1$-norm with three different tail assumptions imposed on $\mathrm{P}$. In particular, analysis will be conducted in situations where the tail of the probability distribution $\mathrm{P}$ has a polynomial decay, exponential decay and disappears, respectively.

\begin{theorem} \label{L1ConvergenceRates} 
For $n \geq 1$, let $\mathcal{A}_{Z,p}$ be a random partition of $B_r$. Moreover, assume that the density $f$ is $\alpha$-H\"{o}lder continuous. We consider the following cases:
\begin{itemize}
\item[(i)]
$\mathrm{P}(B_r^c) \lesssim r^{-\eta d}$ for some $\eta > 0$ and for all $r \geq 1$;
\item[(ii)]
$\mathrm{P}(B_r^c) \lesssim e^{- a r^{\eta}}$ for some $a > 0$, $\eta > 0$ and for all $r \geq 1$;
\item[(iii)]
$\mathrm{P}(B_{r_0}^c) = 0$ for some $r_0 \geq 1$.
\end{itemize}
For the above cases, if $n \geq 1$, and the sequences $p_n$ are of the following forms:
\begin{itemize}
\item[(i)]
$p_n = (n/\log n)^{\frac{2d(\eta+1)+2\alpha}{\alpha(c_T\eta+2)+2d(\eta+1)}}$;
\item[(ii)]
$p_n = (n/\log n)^{\frac{2d}{c_T\alpha+2d}} (\log n)^{\frac{d+\alpha}{\eta} \frac{4d}{c_T\alpha+2d}}$;
\item[(iii)]
$p_n = (n/\log n)^{\frac{2d}{c_T\alpha+2d}}$;
\end{itemize}
then with probability $\mathrm{P}^n \otimes \mathrm{P}_Z$ at least $1 - 2 e^{-\tau}$, there holds
\begin{align*}
\|f_{\mathrm{D},Z} - f\|_1 \leq \varepsilon_n,
\end{align*}
where the convergence rates
\begin{itemize}
\item[(i)]
$\varepsilon_n \lesssim (\log n/n)^{\frac{c_T\alpha\eta}{2\alpha(c_T\eta+2)+4d(\eta+1)}}$;
\item[(ii)]
$\varepsilon_n \lesssim (\log n/n)^{\frac{c_T\alpha}{2c_T\alpha+4d}} (\log n)^{\frac{2d}{\eta} \frac{\alpha+d}{c_T\alpha+2d}}$;
\item[(iii)]
$\varepsilon_n \lesssim (\log n/n)^{\frac{c_T\alpha}{2c_T\alpha+2d}}$.
\end{itemize}
\end{theorem}

\subsection{Results on Convergence Rates under $L_{\infty}$-Norm} \label{subsec::l_infty_convergence}

We now state our main results on the convergence rates of $f_{\mathrm{D},Z}$ to $f$ under $L_{\infty}$-norm.

\begin{theorem} \label{LInftyConvergenceRates} 
	For $n \geq 1$, let $\mathcal{A}_{Z,p_n}$ be a random partition of $B_r$. Moreover, assume that there exists a constant $r \geq 1$ such that $\mathcal{X} \subset B_r \subset \mathbb{R}^d$ and the density function $f$ is $\alpha$-H\"{o}lder continuous with $\|f\|_{\infty} < \infty$. Then for all $n \geq 1$, by choosing
	\begin{align*}
	p_n = (n/\log n)^{\frac{2d}{c_T\alpha+4ad}},
	\end{align*}
	with probability $\mathrm{P}^n \otimes \mathrm{P}_Z$ at least $1 - 3 e^{-\tau}$, there holds
	\begin{align} \label{LInftyRates}
	\|f_{\mathrm{D},Z} - f\|_{\infty}
	\lesssim (\log n/n)^{\frac{c_T\alpha}{2c_T\alpha+8ad}}
	\end{align}
	where $c_T = 0.22$ and $a = 4.33$.
\end{theorem}

\subsection{Convergence Rates for Best-scored Random Density Forest} \label{subsec::l_infty_convergence_forest}

Basing on the results of the base density estimators, we obtain the convergence rates of the best-scored random density forest estimators under $L_1$- and $L_{\infty}$-norm, respectively. Here, we still consider three different tail probability distributions as in Theorem \ref{L1ConvergenceRates}.

\begin{theorem} \label{L1ConvergenceRatesForest} 
	For $n \geq 1$, let $\mathcal{A}_{Z_t}$, $1 \leq t \leq m$ be random partitions of $B_r$ generated by the splitting policies $Z_1, \ldots, Z_m$ respectively. Moreover, assume that the density $f$ is $\alpha$-H\"{o}lder continuous. We consider the following cases:
	\begin{itemize}
		\item[(i)]
		$\mathrm{P}(B_r^c) \lesssim r^{-\eta d}$ for some $\eta > 0$ and for all $r \geq 1$;
		\item[(ii)]
		$\mathrm{P}(B_r^c) \lesssim e^{-ar^{\eta}}$ for some $a > 0$, $\eta > 0$ and for all $r \geq 1$;
		\item[(iii)]
		$\mathrm{P}(B_{r_0}^c) = 0$ for some $r_0 \geq 1$.
	\end{itemize}
	For the above cases, if $n \geq 1$, and the sequences $p_n$ are of the following forms:
	\begin{itemize}
		\item[(i)]
		$p_n = (n/\log n)^{\frac{2d(\eta+1)+2\alpha}{\alpha(c_T\eta+2)+2d(\eta+1)}}$;
		\item[(ii)]
		$p_n = (n/\log n)^{\frac{2d}{c_T\alpha+2d}} (\log n)^{\frac{d+\alpha}{\eta} \frac{4d}{c_T\alpha+2d}}$;
		\item[(iii)]
		$p_n = (n/\log n)^{\frac{2d}{c_T\alpha+2d}}$;
	\end{itemize}
	where the number of splits are the same for each partition in $\{ \mathcal{A}_{Z_t} \}_{t=1}^m$. Then with probability $\mathrm{P}^n \otimes \mathrm{P}_Z$ at least $1 - 2 e^{-\tau}$, there holds
	\begin{align*}
	\|f_{\mathrm{D},Z_{\mathbb{E}}} - f\|_1 \leq \varepsilon_n,
	\end{align*}
	where the convergence rates
	\begin{itemize}
		\item[(i)]
		$\varepsilon_n \lesssim (\log n/n)^{\frac{c_T\alpha\eta}{2\alpha(c_T\eta+2)+4d(\eta+1)}}$;
		\item[(ii)]
		$\varepsilon_n \lesssim (\log n/n)^{\frac{c_T\alpha}{2c_T\alpha+4d}} (\log n)^{\frac{2d}{\eta} \frac{\alpha+d}{c_T\alpha+2d}}$;
		\item[(iii)]
		$\varepsilon_n \lesssim (\log n/n)^{\frac{c_T\alpha}{2c_T\alpha+2d}}$.
	\end{itemize}
\end{theorem}

In the following theorem, we obtain the convergence rates of the best-scored random density forest estimators with respect to the $L_{\infty}$-norm.

\begin{theorem} \label{LInftyConvergenceRatesForest} 
	For $n \geq 1$, let $\mathcal{A}_{Z_t}$, $1 \leq t \leq m$ be random partitions of $B_r$ generated by the splitting policies $Z_1, \ldots, Z_m$ respectively. Moreover, assume that there exists a constant $r \geq 1$ such that $\mathcal{X} \subset B_r \subset \mathbb{R}^d$ and the density function $f$ is $\alpha$-H\"{o}lder continuous with $\|f\|_{\infty} < \infty$. Then for all $n \geq 1$, by choosing the same number of splits
	\begin{align*}
	p_n = (n/\log n)^{\frac{2d}{c_T\alpha+4ad}}
	\end{align*}
	for each partition in $\{ \mathcal{A}_{Z_t} \}_{t=1}^m$. Then with probability $\mathrm{P}^n \otimes \mathrm{P}_Z$ at least $1 - 3 e^{-\tau}$, there holds
	\begin{align*}
	\|f_{\mathrm{D},Z_{\mathrm{E}}} - f\|_{\infty} \lesssim (\log n/n)^{\frac{c_T\alpha}{2c_T\alpha+8ad}}
	\end{align*}
	where $c_T = 0.22$ and $a = 4.33$.
\end{theorem}

\subsection{Comments and Discussions} \label{CandD}

In this section, we present some comments on the obtained theoretical results concerning with the consistency and convergence rates of the best-scored random density tree estimators and the best-scored random density forest estimators.

Trying to alleviate the disadvantages of traditional histogram density estimators, such as their heavy dependence on fixed bin widths and their inevitable discontinuity, we propose to establish the best-scored random density forest estimators based on random partitions and integrate several base estimators to give a smoothed density estimator. Recall that all the estimators presented in this paper are nonparametric density estimators, the criterion measuring their goodness-of-fit does matter. Commonly used measures include $L_1$-distance, $L_2$-distance, $L_{\infty}$-distance. In \cite{devroye1985nonparametric}, authors provide a especially lucid statement of the mathematical attractions of $L_1$ distance: it is always well-defined as a metric on the space of density functions; it is invariant under monotone transformations; and it is proportional to the total variation metric. As for the $L_1$-distance, if we regard $L_1$-distance as the measure of the overall performance, then the $L_1$-distance measures the goodness-of-fit at each point in the feature space, thus it is stronger. We highlight that in our analysis, the convergence rates of the base estimators and the ensemble estimators have all been considered under $L_1$-norm and $L_{\infty}$-norm, respectively.

On the other hand, due to the fact that these best-scored random density tree estimators are all based on random partitions, we should combine the probability distribution of $X$ with the probability distribution of the partition space, which leads to the use of $\mathrm{P}^n \otimes 
\mathrm{P}_Z$ in the analysis of consistency and convergence rates. In virtue of the randomness resided in the partitions, the effective number of splits is smaller than that of the deterministic partitions. As a result, in order to obtain the consistency of a best-scored random density tree estimator, the number of splits should be larger so that the resulting cell sizes can be smaller. Moreover, we establish the convergence rates in the sense of $L_{\infty}$-norm of the best-scored random density forest estimators, namely, $\mathcal{O} ((\log n / n)^{(c_T \alpha)/(2 c_T \alpha + 8 a d)})$, where $c_T = 0.22$ and $a = 4.33$. It is noteworthy that the assumptions needed to establish convergence rate under $L_{\infty}$-norm are not that stronger than assumptions under $L_1$-norm. To be specific, we only need to add two mild assumptions to the original one, which are that the density $f$ should be compactly supported and bounded.

As is mentioned in the introduction, there are a flurry of studies in the literature that address the density estimation problem. Specifically, there are other theoretical studies of histogram density estimations. For example, \cite{lugosi1996consistency} conducts histogram density estimations based on data-dependent partitions where a strong consistency in the sense of $L_1$-norm is obtained under general sufficient conditions. \cite{klemela2009multivariate} presents a multivariate histograms based on data-dependent partitions which are obtained by minimizing a complexity-penalized error criterion. The convergence rates obtained in his study are of the type $\mathcal{O} (\log n/n)$ with respect to the $L_2$-norm under the assumption that function belongs to an anisotropic Besov class. Moreover, it can be regarded as a particular case of our proposal considering its partition process. For kernel density estimation, \cite{jiang2017the} presents that under the assumption of $\alpha$-Hölder continuity, the convergence rates obtained are of the type $\mathcal{O} ((\log n/n)^{\alpha/(2\alpha+d)})$.

\section{Error Analysis} \label{error_analysis}

\subsection{Bounding the Approximation Error Term}

\begin{proposition} \label{ApproximationError}
	For $n \geq 1$, let $\mathcal{A}_{Z,p}$ be a random partition of $B_r$.
	\begin{itemize}
		\item[(i)]
		For any $\varepsilon > 0$, there exists $p_{\varepsilon,\xi} > 0$ such that for any $p \geq p_{\varepsilon,\xi}$, we have
		\begin{align*}
		\|f_{\mathrm{P},Z} - f\|_1 \leq \varepsilon
		\end{align*}
		with probability $\mathrm{P}_Z$ at least $1 - \xi$.
		\item[(ii)]
		If $f$ is $\alpha$-H\"{o}lder continuous, then for all $p \geq (3 r d \xi^{-1} (c/\varepsilon)^{1/\alpha})^{4d/c_T}$ where $c$ is the constant of the $\alpha$-H\"{o}lder continuity, there holds
		\begin{align*}
		\|(f_{\mathrm{P},Z} - f) \eins_{B_r}\|_{\infty} \leq \varepsilon
		\end{align*}
		with probability $\mathrm{P}_Z$ at least $1 - \xi$.
	\end{itemize}
\end{proposition}

\begin{proposition} \label{L1LInftyRelation}
	For $n \geq 1$, let $\mathcal{A}_{Z,p}$ be a random partition of $B_r$. Then, for $r \geq 1$, we have
	\begin{align*}
	\|f_{\mathrm{P},Z} - f\|_1 
	\leq 2^d r^d \|(f_{\mathrm{P},Z} - f) \eins_{B_r}\|_{\infty} + 2 \mathrm{P}(B_r^c).
	\end{align*}
\end{proposition}

\subsection{Bounding the Estimation Error Term}

\subsubsection{A Fundamental Lemma}

The following lemma shows that both of the $\|\cdot\|_1$ and $\|\cdot\|_{\infty}$-distance between $f_{\mathrm{D},Z}$ and $f_{\mathrm{P},Z}$ can be estimated by the quantities $|\mathbb{E}_{\mathrm{D}} \eins_{A_j} - \mathbb{E}_{\mathrm{P}} \eins_{A_j}|$.

\begin{lemma} \label{FundamentalLemma}
	Let $\mathcal{A}_{Z,p}$ be a random partition of $B_r$. Then the following equalities hold:
	\begin{itemize}
		\item[(i)]
		$\displaystyle \|f_{\mathrm{D},Z} - f_{\mathrm{P},Z}\|_1 = \sum_{j=0}^p |\mathbb{E}_{\mathrm{D}} \eins_{A_j} - \mathbb{E}_{\mathrm{P}} \eins_{A_j}| + |\mathbb{E}_{\mathrm{D}} \eins_{B_r^c} - \mathbb{E}_{\mathrm{P}} \eins_{B_r^c}|$.
		\item[(ii)]
		$\displaystyle \|f_{\mathrm{D},Z} - f_{\mathrm{P},Z}\|_{\infty} = \sup_{j \in \{ 0, \ldots, p+1 \}} \frac{1}{\mu(A_j)} |\mathbb{E}_{\mathrm{D}} \eins_{A_j} - \mathbb{E}_{\mathrm{P}} \eins_{A_j}|$, where we denote $A_{p+1} := B_r^c$.
	\end{itemize}
\end{lemma}

\subsubsection{Bounding the Capacity of the Function Set $\mathcal{T}$}

\begin{definition}[Covering Numbers]
	Let $(X, d)$ be a metric space, $A \subset X$ and $\varepsilon > 0$. We call $A' \subset A$ an $\varepsilon$-net of $A$ if for all $x \in A$ there exists an $x' \in A'$ such that $d(x, x') \leq \varepsilon$. Moreover, the $\varepsilon$-covering number of $A$ is defined as
	\begin{align*}
	\mathcal{N}(A, d, \varepsilon)
	= \inf \biggl\{ n \geq 1 : \exists x_1, \ldots, x_n \in X \text{ such that } A \subset \bigcup_{i=1}^n B_d(x_i, \varepsilon) \biggr\},
	\end{align*}
	where $B_d(x, \varepsilon)$ denotes the closed ball in $X$ centered at $x$ with radius $\varepsilon$.
\end{definition}

Let $p \in \mathbb{N}$ be fixed. Let $\pi_p$ be a partition of $\mathcal{X}$ with number of splits $p$ and $\pi_{(p)}$ denote the collection of all partitions $\pi_p$. Further, we define
\begin{align} \label{Bp}
\mathcal{B}_p := \biggl\{ B : B = \bigcup_{j \in J} A_j, J \subset \{ 0, 1, \ldots, p \}, A_j \in \pi_p \subset \pi_{(p)} \biggr\}.
\end{align}

Let $B$ be a class of subsets of $\mathcal{X}$ and $A \subset \mathcal{X}$ be a finite set. The trace of $B$ on $A$ is defined by $\{ B \cap A : B \in \mathcal{B} \}$. Its cardinality is denoted by $\triangle^{\mathcal{B}}(A)$. We say that $\mathcal{B}$ shatters $A$ if $\triangle^{\mathcal{B}}(A) = 2^{\#(A)}$, that is, if for every $A' \subset A$, there exists a $B \subset \mathcal{B}$ such that $A' = B \cap A.$ For $k \in \mathbb{N}$, let
\begin{align*}
m^{\mathcal{B}}(k) := \sup_{A \subset \mathcal{X}, \, \#(A) = k} \triangle^{\mathcal{B}}(A).
\end{align*}
Then, the set $\mathcal{B}$ is a Vapnik-Cervonenkis (VC) class if there exists $k < \infty$ such that $m^{\mathcal{B}}(k) < 2^k$ and the minimal of such $k$ is called the $VC$ index of $\mathcal{B}$, and abbreviated as $\mathrm{VC}(\mathcal{B})$.

\begin{lemma} \label{VCindex}
	The VC index of $\mathcal{B}_p$ can be upper bounded by $d p + 2$.
\end{lemma}

Let $\mathcal{B}$ be a class of subsets of $\mathcal{X}$, denote $\eins_{\mathcal{B}}$ as the collection of the indicator functions of all $B \in \mathcal{B}$, that is, $\eins_{\mathcal{B}} := \{ \eins_B : B \in \mathcal{B} \}$. Moreover, as usual, for any probability measure $Q$, $L_1(Q)$ is denoted as the $L_1$ space with respect to $Q$ equipped with the norm $\|\cdot\|_{L_1(Q)}$.

\begin{lemma} \label{ScriptBpCoveringNumber}
	Let $\mathcal{B}_p$ be defined as in \eqref{Bp}. Then, for all $0 < \varepsilon < 1$, there exists a universal constant $K$ such that
	\begin{align} \label{BpCoveringNumber}
	\mathcal{N}(\eins_{\mathcal{B}_p}, \|\cdot\|_{L_1(Q)}, \varepsilon)
	\leq K (d p + 2) (4 e)^{d p + 2} \biggl( \frac{1}{\varepsilon} \biggr)^{d p + 1}
	\end{align}
	holds for any probability measure $Q$.
\end{lemma}

\subsubsection{Oracle Inequalities under $L_1$-Norm, and $L_{\infty}$-Norm}

\begin{proposition} \label{OracleInequalityL1}
	Let $\mathcal{A}_{Z,p}$ be a random partition of $B_r$. Then, for all $r \geq 1$, $n \geq 1$, and $\tau > 0$, with probability $\mathrm{P}^n$ at least $1 - e^{-\tau}$, there holds
	\begin{align*}
	\|f_{\mathrm{D},Z} - f_{\mathrm{P},Z}\|_1
	\leq \sqrt{\frac{9(\tau+15dp\log n)}{2n}} + \frac{2(\tau+15dp\log n)}{n} + \frac{4}{n}.
	\end{align*}
\end{proposition}

\begin{proposition} \label{OracleInequalityLInfty}
	Let $\mathcal{A}_{Z,p}$ be a random partition of $B_r$. Assume that there exists a constant $r \geq 1$ such that $\mathcal{X} \subset B_r \subset \mathbb{R}^d$ and the density function $f$ satisfies $\|f\|_{\infty} < \infty$. Then for all $\tau > 0$ and $n \geq 1$, with probability $\mathrm{P}^n \otimes \mathrm{P}_Z$ at least $1 - 2 e^{-\tau}$, there holds
	\begin{align}
	\|f_{\mathrm{D},Z} - f_{\mathrm{P},Z}\|_{\infty}
	& \leq \sqrt{\frac{8 \|f\|_{\infty} p^{2a} ((8d+1) \tau + 23 \log n + 8 a d \log p)}{n \mu(B_r) e^{-2\tau}}}
	\nonumber\\
	& \phantom{=}
	+ \frac{8 p^{2a}((8d+1) \tau + 23 \log n + 8 a d \log p)}{3 n \mu(B_r) e^{-2\tau}} + \frac{2}{n}.
	\label{SampleErrorLInftyBound}
	\end{align}
	\end{proposition}
	
	\section{Numerical Experiments} \label{sec::numerical_experiments}

	In this section, we present the computational experiments that we have carried out. Aiming at obtaining more efficient partition performance, we improve the purely random splitting criteria to new ones named as the adaptive random splitting criteria in Section \ref{sec::subsec::adaptivemethod}. Concerning with the fact that the partitions currently discussed are all axis-parallel and their induced density estimators may not be accurate enough for some cases, we extend them to the oblique random partitions in Section \ref{sec::subsec::adaptiveobliquepartition}. Based on the adaptive random partitions, we construct our best-scored random density forest in Section \ref{sec::subsec::BDE}. Then we compare our approach with other proposals illustrated in Section \ref{sec::subsec::experisetup} both on synthetic data in Section \ref{sec::subsec::syntheticdata} and real data in Section \ref{sec::subsec::realdata}.

\subsection{Improvement with Adaptive Method} \label{sec::subsec::adaptivemethod}

It is worth pointing out one crucial fact that the purely random trees may face the dilemma, the effective number of splits being relatively small. The reason for this phenomenon is that the purely random splitting criteria make no use of the sample information. Therefore, we propose an adaptive splitting method efficiently taking sample information into consideration. Since we have only discussed on partitions that are axis-parallel, the corresponding new criterion is called the \emph{adaptive axis-parallel random splitting criterion}.

Recall that for the axis-parallel purely random partition, $L_i$ in the random vector $Q_i := (L_i, R_i, S_i)$ denotes the randomly chosen cell to be split at the $i$th step of the tree construction. However, on account that this choice of $L_i$ does not make any use of the sample information, it may suffer from over-splitting in sample-sparse areas and under-splitting in sample-dense areas. Hence, we propose that when choosing a to-be-split cell, we first randomly select $t$ samples from the training data set and find out which cell most of these samples fall in. Then, we pick up this cell as $L_i$. This idea comes from the fact that when randomly picking sample points from the whole training data set, cells with more samples are more possible to be selected while cells with fewer samples are less likely to be chosen. To mention, $t$ as a hyperparameter can be tuned according to specific conditions.

We mention that the ``adaptive'' here can be illustrated from the perspective of partition results. Splits are conducted according to the sample distribution where sample-dense areas will be split more, while sample-sparse areas will be split less. In this way, we develop an adaptive axis-parallel random partition.

\subsection{An Adaptive Oblique Partition} \label{sec::subsec::adaptiveobliquepartition}

So far we have considered cases where each split in the partition process is conducted from only one dimension of the feature space, i.e. all splits are axis-parallel. However, in order to achieve better experiments results, we now extend our adaptive axis-parallel random splitting criterion in Section \ref{sec::subsec::adaptivemethod} to a more advanced criterion where splits can be oblique and the locations of them are data-driven. This new criterion is called the \emph{adaptive oblique random splitting criterion}. To mention, the ``adaptive'' implies that each step in the construction procedure is adaptive to the sample distribution. The competitiveness of this new splitting criterion can be verified by the experimental analysis on real data.

Here, we illustrate a possible construction approach of one tree by following the adaptive oblique random splitting criterion. Still, some randomizing variables are needed for a clear description. The oblique partition process at the $i$th step can be represented by $O_i := (L_i, G_i, W_i)$. When conducting the adaptive partition method, we first randomly select $t$ samples from the training data set and each of these sample points is certain to fall into one of the cells formed at the $(i - 1)$th step. Therefore, we can find out which cell most of these samples fall in and choose this cell as the to-be-split cell $L_i$. Secondly, the coordinates of samples falling into cell $L_i$ in $t$ samples are recorded and thus the barycenter $G_i$ of cell $L_i$ can be substituted by the centroid $X_i^c$ of these samples in experiments. Thirdly, since we follow an oblique splitting rule, the split performed on $L_i$ is actually a part of the chosen hyperplane $W_i^T x + b_i = 0$, $x \in L_i$. For the experimental convenience, we set normal vectors of hyperplanes $\{ W_i, i \in \mathbb{N}_+ \}$ to be independent and identically distributed from $\mathcal{U}[-1, 1]^d$ and $b_i := - W_i^T X_i^c$. Till now, we finish the construction of the $i$th step, and by following this procedure recursively, we are able to establish a random tree with oblique partitions. 

It can be clearly observed from the establishment of one tree estimator \eqref{RandomDensityTreeQ} in the best-scored random density forest that after obliquely partitioning the feature space into non-overlapping cells which are irregular polyhedrons, we are in need of the volume of each cell. In general, the method of computing the volume of an irregular polyhedron is mainly to decompose the polyhedron into a plurality of solvable polyhedrons. However, this approach is not a wise choice for its high computational complexity especially when the dimension of the space where the polyhedron is embedded is high. Take the volume computation of one polyhedron for example, we need to determine the specific coordinates of each vertex of the polyhedron to select an appropriate polyhedron decomposition method. Moreover, this approach provides an exact value for the polyhedral volume. Exact value of volume is not a must since our density estimator is itself an approximate to the ground-truth density. In fact, good approximations of polyhedron volumes are enough for our algorithm when carrying out experiments. Therefore, we employ another well-known volume estimation method which is the Monte Carlo method relying on repeated random sampling to obtain numerical results. The specific procedure can be stated as follows: First of all, we find the smallest hypercube in the feature space that contains all the training data. The side length of this hypercube in each dimension is the difference between the maximum and minimum values of the training sample's coordinates in this dimension, so that the volume of the hypercube can be easily obtained. Secondly, we generate $2,000$ points by default according to the uniform distribution on the area where the hypercube is located, and record the ratio of the number of points falling into each cell to the number of all points, or called frequency. Lastly, the volume of each cell is calculated by the volume of the hypercube multiplied by the corresponding frequency. On account that the Monte Carlo method is easy to operate and gives a good estimate of the volume of any irregular polyhedron in any dimension, we employ this approach for our algorithm under oblique partitions.

\subsection{Best-scored Density Estimators} \label{sec::subsec::BDE}

Having demonstrated how to perform an adaptive splitting criterion for both axis-parallel and oblique partitions, we now come to the discussion on details of the construction of one best-scored random tree. First of all, we generate $k$ $p$-split adaptive random partitions and our main purpose is to select one partition with the best density estimation performance out of the $k$ candidates by a $10$-fold cross-validation. Now take the first round of the $10$-fold cross-validation as an example. For each of the $k$ partition candidates, training set of the cross-validation is used to give weight to each cell of that partition according to \eqref{RandomDensityTreeQ} and the corresponding validation error of the partition based on the validation set is then computed. After traversing all ten rounds, we are able to obtain the average validation error for each of the $k$ partition candidates, and the one with the smallest average validation error is the exact partition for one tree. Based on this selected partition, we give weight to each cell in accordance with \eqref{RandomDensityTreeQ} based on the whole training set. By repeating the above establishment procedure for $m$ times, we are able to obtain a best-scored random forest for density estimation containing $m$ trees.

\subsection{Experimental Setup} \label{sec::subsec::experisetup}

In our experiments, comparisons are conducted among our axis-parallel best-scored random density forest (BRDF-AP), oblique best-scored random density forest (BRDF-OB) and other effective density estimation methods, which are
\begin{itemize}
\item
KDE: the kernel density estimators \cite{parzen1962on,rosenblatt1956remarks} where we take the Gaussian kernel.
\item
DHT: the discrete histogram transform \cite{lopezrubio2014a} is a nonparametric approach based on the integration of several multivariate histograms which are computed over affine transformations of the training data.
\item
rNADE: the real-valued neural autoregressive distribution estimation models \cite{uria2013rnade,uria2016neural}, which are neural network architectures applied to the problem of unsupervised distribution and density estimation.
\end{itemize}
It is worth pointing out that for KDE, we utilize the Python-package SciPy with default settings, for DHT, \cite{lopezrubio2014a} provides the codes in Matlab and for rNADE, \cite{uria2016neural} also provides codes in Python. All the following experiments are performed on computer equipped with 2.7 GHz Intel Core i7 processor, 16 GB RAM.

In order to provide a quantitative comparisons of options, we adopt the following Mean Absolute Error:
\begin{align} \label{MAE}
\textit{MAE}(\hat{f}) = \frac{1}{M} \sum_{j=1}^M |\hat{f}(x_j) - f(x_j)|,
\end{align}
where $x_1, \ldots, x_M$ are test samples. It is mainly used in cases where the true density function is known. Though \emph{MAE} can not be used in tests of real data whose true density is unknown, it is still especially suitable for synthetic data experiments.

Another effective measure of estimation accuracy, especially when facing real data, is measured over $M$ test samples, which is given by the Average Negative Log-Likelihood:
\begin{align} \label{ANLL}
\textit{ANLL}(\hat{f}) = - \frac{1}{M} \sum_{j=1}^M \log \hat{f}(x_j),
\end{align}
where $\hat{f}(x_j)$ represents the estimated probability density for the test sample $x_j$ and the lower the \emph{ANLL} is, the better estimation we obtain. To mention, we employ $10$-fold cross validated \emph{ANLL} as our test error. One thing that has been attached great importance is that whenever the estimation function $\hat{f}(x)$ of any samples returns zero weight, \emph{ANLL} will go to infinity, which is undesirable. Therefore, we substitute all density estimation $\hat{f}(x)$ with $\hat{f}(x) + \epsilon$, where $\epsilon$ is a infinitesimal number that can be recognized by the computer which can be obtained by function {\tt numpy.spacing(1)} in Python, or constant {\tt eps} in Matlab. Consequently, we come to a desirable state where one bad sample point (with zero estimated probability) will not harm the whole good \emph{ANLL} much.

\subsection{Synthetic Data} \label{sec::subsec::syntheticdata}

In this subsection, we start by applying our BRDF and other above mentioned density estimation methods on several artificial examples. In order to give a more comprehensive understanding of our algorithm architecture illustrated in Section \ref{sec::methodology}, we first consider BRDF with axis-parallel partition (BRDF-AP) here. To be specific, we base the simulations on two different types of distribution construction approaches with each type generating four toy examples with dimension $d = 1, 2, 3, 5$, respectively. To notify, the premise of constructing data sets is that we assume that the components $X_1, \ldots, X_d$ of the random vector $X = (X_1, \ldots, X_d) \in [0, 1]^d$ are independent of each other with identical margin distribution. Therefore, we only present the margin density in the following descriptions.
\begin{itemize}
	\item 
	Type I: $f_{\mathrm{margin}} = 0.3 * \mathrm{uniform}(0.7, 1) + 0.7 * \mathrm{uniform}(0, 0.4)$,
	\item 
	Type II: $f_{\mathrm{margin}} = 0.3 * \mathrm{beta}(11, 20) + 0.7 * \mathrm{uniform}(0.5, 1)$.
\end{itemize}
It can be apparently seen that the construction of example of Type I is piecewise constant and example of Type II is based on mixture models with beta distributions and uniform distributions. We emphasize that the densities of both Types I and II datasets are compactly supported and bounded. In order to give clear descriptions of the distributions, we give the 3D plots of the above two types of distributions with dimension $d = 2$ shown in Figure \ref{SynthDataPlot}.

\begin{figure}[htbp]
	\centering
	\captionsetup{justification=centering}
	\hspace{-10mm}
	\begin{minipage}[t]{0.5\linewidth}
		\centering 
		\includegraphics[height=6cm,width=9.5cm]{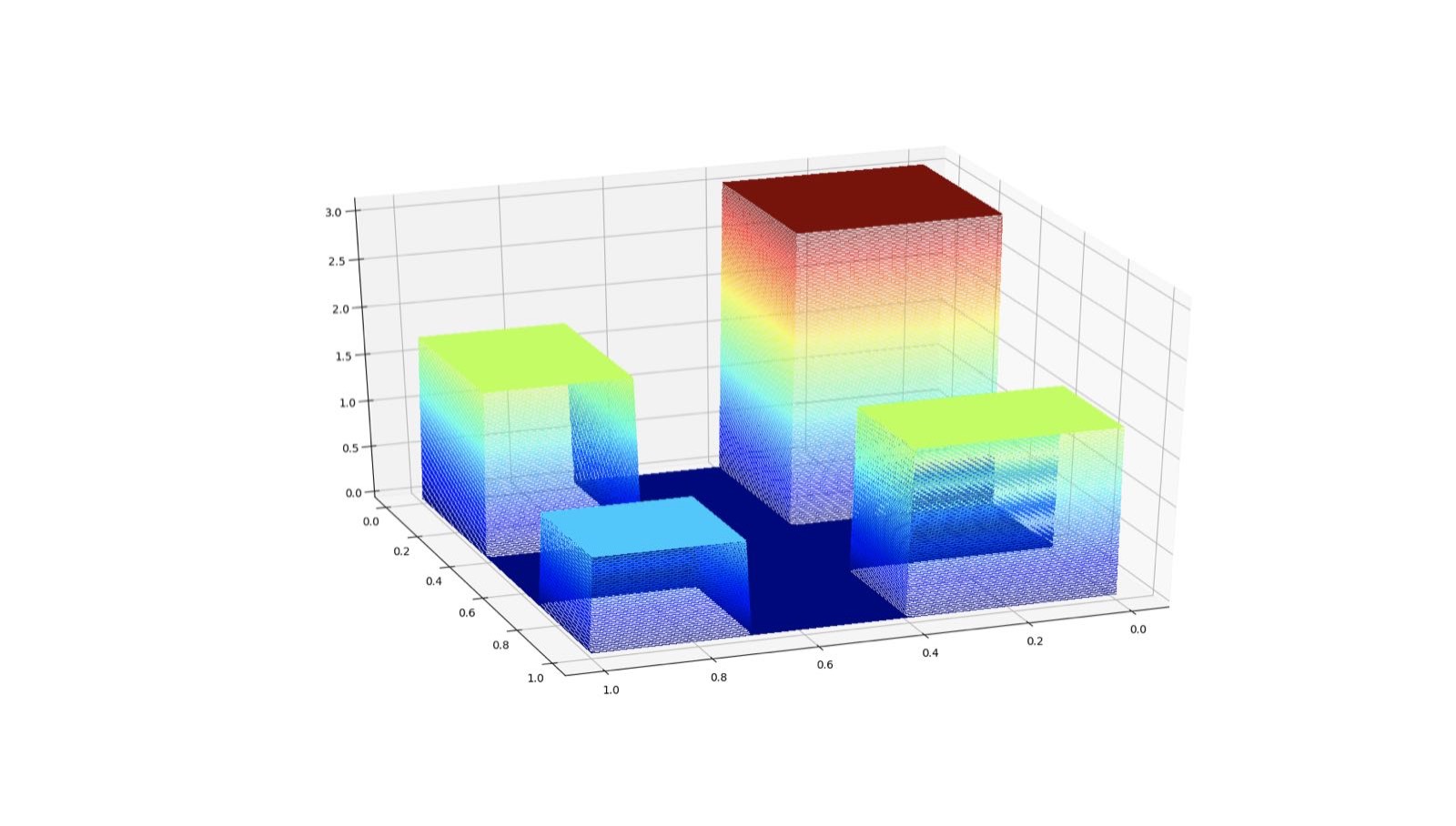} 
		\label{fig:side:a} 
	\end{minipage} 
	\begin{minipage}[t]{0.5\linewidth}
		\centering 
		\includegraphics[height=6cm,width=9.5cm]{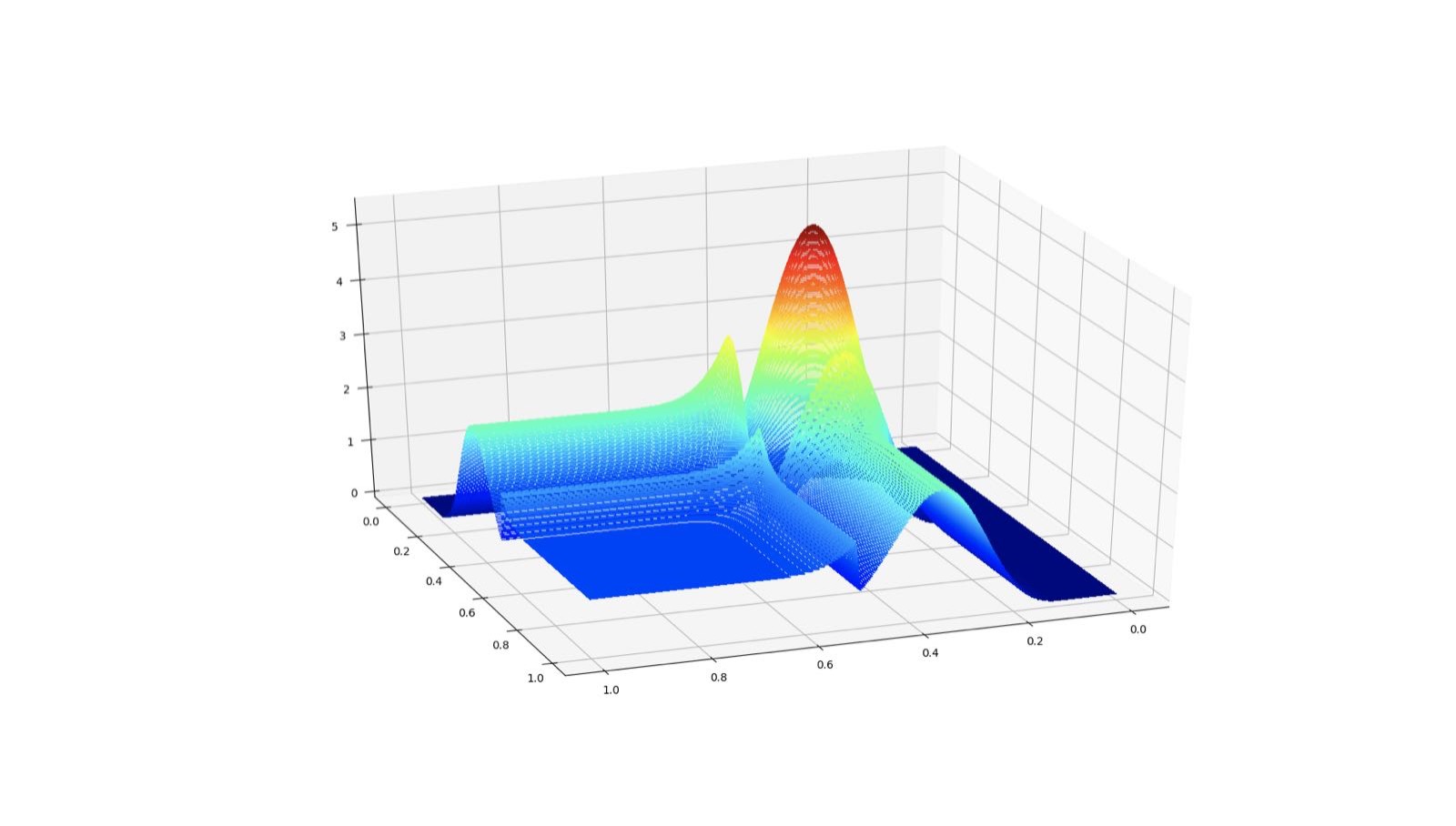} 
		\label{fig:side:b} 
	\end{minipage} 
	\vspace{-12mm}
	\caption{3D plots of the true probability densities with $d = 2$\\ of datasets of ''Type I" (left), ''Type II" (right).}
	\label{SynthDataPlot}
\end{figure}

Table \ref{tab::1} summarizes the average \emph{MAE} performances of our model, KDE, DHT and rNADE. All experiments presented are repeated for $20$ times, and we present the average results here. It can be apparently observed from the Table 1 that our BRDF-AP has the best performances w.r.t. \emph{MAE} on almost all data sets, which further demonstrates the effectiveness of the algorithm.

\begin{table*}
	\setlength{\tabcolsep}{11pt}
	\centering
	\captionsetup{justification=centering}
	\caption{Average \textit{MAE} over the Two Types of Synthetic Data Sets}
	\label{RMSETable}
	\begin{tabular}{cccccc}
		\toprule
		Data Set  & \textit{d} & BRDF-AP & KDE & DHT & rNADE  \\
		\hline 
		\hline
		\multirow{4}*{\text{Type I}} & 1 & \textbf{0.11} & 0.26 & 0.16 & 0.48\\
		\cline{2-6}
		\multicolumn{1}{c}{}& 2 & \textbf{0.39} & 0.88 & 0.51 & 1.07\\
		\cline{2-6}
		\multicolumn{1}{c}{}& 3 & \textbf{1.19} & 2.05 & 1.25 & 1.95\\
		\cline{2-6}
		\multicolumn{1}{c}{}& 5 & \textbf{5.14} & 6.95 & 5.27 & 5.72\\
		\hline
		\multirow{4}*{\text{Type II}} & 1 & 0.28 & 0.17 & \textbf{0.14} & 0.57\\
		\cline{2-6}
		\multicolumn{1}{c}{}& 2 & \textbf{0.35} & 0.47 & 0.45 & 1.86\\
		\cline{2-6}
		\multicolumn{1}{c}{}& 3 & \textbf{0.76} & 0.96 & 1.08 & 1.87\\
		\cline{2-6}
		\multicolumn{1}{c}{}& 5 & \textbf{2.79} & 2.90 & 3.98 & 4.41\\
		\bottomrule	
	\end{tabular}
	\begin{tablenotes}
		\footnotesize
		\item ~~~~~~~~~~~~~~~~~{*} The best results are marked in \textbf{bold}.
	\end{tablenotes}
	\label{tab::1}
\end{table*}

\subsection{Real Data Analysis} \label{sec::subsec::realdata}

In order to obtain better experimental performances for real data analysis, we adopt BRDF with both axis-parallel and oblique partitions. Empirical comparisons on \emph{ANLL} and training time among BRDF-AP, BRDF-OB, KDE, DHT and rNADE are based on data sets from the UCI Repository of machine learning databases: \emph{Wine quality data set}, \emph{Parkinsons telemonitoring data set} and \emph{Ionosphere data set}. Since data contained in the Wine quality data set are actually twofold, which are \emph{Red wine} and \emph{White wine}, we conduct each data set separately.

\begin{table*}[h] 
	\setlength{\tabcolsep}{9pt}
	\centering
	\captionsetup{justification=centering}
	\caption{\footnotesize{Average \textit{ANLL} and Training Time (in Seconds) over the Test Sets for Four UCI Data Sets}}
	\label{RMSETable}	
	\resizebox{\textwidth}{13mm}{
		\begin{tabular}{cccccccccccc}
			\toprule
			\multirow{2}*{\text{Datasets}}
			& \multirow{2}*{$(n,d)$} 
			& \multicolumn{2}{c}{\text{BRDF-AP}}
			& \multicolumn{2}{c}{\text{BRDF-OB}} 
			& \multicolumn{2}{c}{\text{KDE}}
			& \multicolumn{2}{c}{\text{DHT}}
			& \multicolumn{2}{c}{\text{rNADE}} \\	
			\cline{3-12}
			\multicolumn{2}{c}{}&\textit{ANLL}&\text{Time} &\textit{ANLL}&\text{Time}
			&\textit{ANLL}&\text{Time} 
			&\textit{ANLL}&\text{Time}  &\textit{ANLL}&\text{Time} \\
			\hline 
			\hline
			\text{Parkinsons} & (5875,15) & 0.51 & 3.25 & \textbf{-0.67} & \textbf{2.32} & 8.27 & 3.70 & 11.22 & 9.41 & 5.26 & 35.93\\
			\hline
			\text{Ionosphere} & (351,32) & \textbf{0.06} & 19.61 & 4.28 & 18.81 & 24.36 & \textbf{0.13}  & 26.20 & 0.47 & 9.48 & 52.43\\
			\hline
			\text{Red wine} & (1599, 11) & -0.03  & 1.28 & \textbf{-0.19} & 1.13 & 11.57 & \textbf{0.31} & 11.63 & 0.90 & 10.66 & 41.02\\ 
			\hline
			\text{White wine} & (4898, 11) & -0.61 & 2.40 & \textbf{-1.20} & 1.91 & 11.49 & \textbf{1.90} & 11.97 & 4.56 & 10.84 & 39.99\\ 
			\bottomrule	
	\end{tabular}}
	\begin{tablenotes}
		\footnotesize
		\item{*} The best results are marked in \textbf{bold}.
	\end{tablenotes}
	\label{tab::1}
\end{table*}

Some data preprocessing approaches are in need for the following analysis. According to \cite{tang2012deep}, not only discrete-valued attributes but an attribute from every pair with a Pearson correlation coefficient greater than $0.98$ are eliminated. Moreover, all results are reported on the normalized data where each dimension of the data is subtracted its training-subset sample mean and divided by its standard deviation.

The quantitative results are given in Table 2. In this table, the sample size is denoted by $n$ and the data dimensionality is denoted by $d$. Careful observations will find that both our BRDF-AP and BRDF-OB have significantly smaller \emph{ANLLs} than any other standard method on all four data sets. This advantage in estimation accuracy may be attributed to both the general architecture of random forest and the BRDF's unique property of having many tunable hyperparameters. In particular, BRDF-OB has even better estimation performance than BRDF-AP, which demonstrates the effectiveness in employing oblique partitions for forest construction. When focusing on the Ionosphere data set, we find that it takes BRDF-AP and BRDF-OB the longest time to train the models, though Ionosphere has the smallest sample size among all data sets. This phenomenon comes from the fact that we take $m = 100$ to obtain good \emph{ANLLs} on Ionosphere while smaller $m$ is enough to provide satisfying results on other data sets. We should be aware that larger $m$ will bring smoother density estimators and therefore better \emph{ANLL} results, but it also takes longer time, which reflects a trade-off between estimation accuracy and the corresponding training time w.r.t $m$. As is acknowledged, the computation time of KDE grows linearly with the sample size. While compared to KDE, the table shows that the computation time of our BRDF grows much slower than it. The above analysis illustrates the observation that when the sample size is $5,875$, the training time of BRDF-OB is shorter than that of KDE. To mention, the training time can be further shortened if we employ the parallel computing. Moreover, our BRDF is much more resistant to the curse of dimensionality than any other strategies in comparison. 

From a holistic perspective, both our BRDF-AP and BRDF-OB have shown significant advantages over other standard density estimators in real data analysis.

\section{Proofs} \label{sec::proofs}

To prove Proposition \ref{ApproximationError}, we need the following result which follows from Lemma 6.2 in \cite{devroye1986a}.

\begin{lemma} \label{SaturationLevel} 
	For a binary search tree with $n$ nodes, denote the saturation level $S_n$ as the number of full levels of nodes in the tree. Then for $k \geq 1$ and $\log n > k + \log(k + 1)$, there holds
	\begin{align*}
	\mathrm{P}(S_n < k + 1)
	\leq \biggl( \frac{k+1}{n} \biggr) \biggl( \frac{2e}{k} \log \biggl( \frac{n}{k+1} \biggr) \biggr)^k.
	\end{align*}
\end{lemma}

\begin{proof}[of Proposition \ref{ApproximationError}] 
	\emph{(i)} 
	Since the space of continuous and compactly supported functions $C_c(\mathbb{R}^d)$ is dense in $L_1(\mathbb{R}^d)$, we can find $\tilde{f} \in C_c(\mathbb{R}^d)$ such that
	\begin{align} \label{CompactApprox}
	\|f - \tilde{f}\|_1 \leq \varepsilon/3,
	\qquad \qquad 
	\forall \varepsilon > 0.
	\end{align}
	Since $\tilde{f}$ has a compact support, there exists a $r > 0$ such that $\mathrm{supp}(\tilde{f}) \subset B_r$ and $\mu(B_r) > 0$. Moreover, $\tilde{f}$ is uniformly continuous, since it is continuous and $\mathrm{supp}(\tilde{f})$ is compact. This implies that there exists a $\delta \in (0, 1]$ such that if $\|x - x'\|_1 \leq \delta$, then we have
	\begin{align} \label{ftildeUniformlyContinuity}
	|\tilde{f}(x) - \tilde{f}(x')|
	\leq \frac{\varepsilon}{3 \mu(B_r)}.
	\end{align}
	We define $\bar{f} : \mathbb{R}^d \to \mathbb{R}$ by
	\begin{align} \label{fbar}
	\bar{f}(x) := 
	\begin{cases}
	\frac{1}{\mu(A(x))} \int_{A(x)} \tilde{f}(x') \, d\mu(x'), & \text{ if } \mu(A(x)) > 0, 
	\\
	0, & \text{ otherwise }
	\end{cases}
	\end{align}
	where $A(x)$ is the cell which $x$ falls in of the specific partition $\mathcal{A}_{Z,p}$. Then, for any $\varepsilon > 0$,
	\eqref{CompactApprox} implies that
	\begin{align}
	\|f - f_{\mathrm{P},Z}\|_1
	& \leq \|f - \tilde{f}\|_1 + \|\tilde{f} - \bar{f}\|_1 + \|\bar{f} - f_{\mathrm{P},Z}\|_1
	\nonumber\\
	& \leq \varepsilon/3 + \|\tilde{f} - \bar{f}\|_1 + \|\bar{f} - f_{\mathrm{P},Z}\|_1.
	\label{L1ErrorDecomposition}
	\end{align}
	If $x \in B_r^c$, then we have $\tilde{f}(x) = 0$. Moreover, if $\mu(A(x)) > 0$, then
	\begin{align*}
	\bar{f}(x) := \frac{1}{\mu(A(x))} \int_{A(x)} \tilde{f}(x') \, d\mu(x') = 0.
	\end{align*}
	Otherwise if $\mu(A(x)) = 0$, then $\bar{f}(x) = 0$ by the definition. Therefore, we obtain
	\begin{align*}
	\|\tilde{f} - \bar{f}\|_1
	= \int_{B_r} |\tilde{f}(x) - \bar{f}(x)| \, d\mu(x).
	\end{align*}
	For $x \in B_r$ with $\mu(A(x)) > 0$, there holds
	\begin{align*}
	|\tilde{f}(x) - \bar{f}(x)|
	& = \biggl| \frac{1}{\mu(A(x))} \int_{A(x)} \tilde{f}(x) \, d\mu(x')
	- \frac{1}{\mu(A(x))} \int_{A(x)} \tilde{f}(x') \, d\mu(x') \biggr|
	\\
	& \leq \frac{1}{\mu(A(x))} \int_{A(x)} |\tilde{f}(x) - \tilde{f}(x')| \, d\mu(x').
	\end{align*}
	In the following proof, in order to describe the randomness of the partition, we should first give the definition of the diameter of a cell by 
	$\mathrm{diam}(A) = \sum_{i=1}^d V_i(A)$, where $V_i(A)$ denotes the length of the $i$-th dimension of a rectangle cell $A$. Then by Markov's inequality, we obtain
	\begin{align}
	\mathrm{P}_Z ( \forall A \in \mathcal{A}_{Z,p} : \mathrm{diam}(A) \leq h)
	& = \mathrm{P}_Z \biggl( \max_{A \in \mathcal{A}_{Z,p}} \mathrm{diam}(A) \leq h \biggr)
	\nonumber\\
	& \geq 1 - h^{-1} \mathbb{E}_Z \biggl( \max_{A \in \mathcal{A}_{Z,p}} \mathrm{diam}(A) \biggr)
	\nonumber\\
	& = 1 - h^{-1} \mathbb{E}_Z \biggl( \max_{A \in \mathcal{A}_{Z,p}} \sum_{i=1}^d V_i(A) \biggr)
	\nonumber\\
	& \geq 1 - h^{-1} \sum_{i=1}^d \mathbb{E}_Z \biggl( \max_{A \in \mathcal{A}_{Z,p}} V_i(A) \biggr)
	\label{diameterA}
	\end{align}
	Recall that $Z$ is defined by $(Q_1, \ldots, Q_p, \ldots)$ where $Q_i = (L_i, R_i, S_i)$, $i = 0, 1, \ldots$ in Section \ref{sec::subsec::PRP}. This shows that the randomness of $Z$ actually results from three aspects: randomly selecting nodes, randomly picking dimensions, randomly determining cut points. The above analysis of the variable $Z$ describes the exact constructing process of the tree entirety. In order to ensure the feasibility of the calculation of expectation respect to $Z$ in \eqref{diameterA}, we need to conduct analysis supposing that the tree is well-established. In particular, for each dimension, we only consider one cell that has the longest side length in its respective dimension. To mention, due to the symmetry of dimensions, it suffices to first focus on one dimension, e.g. the $i$th dimension and we denote the length of the $i$-th dimension of the corresponding cell as $\max_{A \in \mathcal{A}_{Z,p}} V_i(A) =: V_Z$. To calculate $\mathbb{E}_Z(V_Z)$, we do not have to know the exact constructing procedures of the tree entirety. Instead, we still consider from three aspects which is intrinsically corresponding to above, but from a different view: the total number of splits that generates that specific rectangle cell during the construction, $T_Z$; the number of splits which come from the $i$th dimension in $T_Z$, $K_Z$ and $K_Z$ follows the binomial distribution $\mathcal{B}(T_Z, 1/d)$; and proportional factors $U_1, U_2, \ldots, U_{K_Z}$ which are independent and identically distributed from $\mathcal{U}[0, 1]$. According to the above statements, we come to the conclusion that the expectation with regard to $Z$ can be decomposed as $\mathbb{E}_Z = \mathbb{E}_{T_Z} \mathbb{E}_{K_Z|T_Z} \mathbb{E}_{U_1 \ldots U_{K_Z}|K_Z}$. Therefore, the expectation in the last step in \eqref{diameterA} can be further analysed as follows:
	\begin{align*}
	\mathbb{E}_Z V_Z 
	& \leq \mathbb{E}_{T_Z} \biggl( \mathbb{E}_{K_Z} \biggl( \mathbb{E}_{U_1 \ldots U_{K_Z}} \biggl( 2 r \cdot \prod_{j=1}^{K_Z} \max \{ U_j, 1 - U_j \} \Big| K_Z \bigg) \Big| T_Z \biggr) \biggr)
	\\
	& = 2 r \cdot \mathbb{E}_{T_Z} \Bigl( \mathbb{E}_{K_Z} \Bigl( \Bigl( \mathbb{E}_U (\max\{ U, 1 - U \})^{K_Z} \Big| T_Z \Bigr) \Bigr)
	= 2 r \cdot \mathbb{E}_{T_Z} \Bigl( \mathbb{E}_{K_Z} \Bigl( (3/4)^{K_Z} | T_Z \Bigr) \Bigr)
	\\
	& = 2 r \cdot \mathbb{E}_{T_Z} \biggl( \sum_{K_Z=1}^{T_Z} {T_Z \choose K_Z} 
	\biggl( \frac{3}{4} \biggr)^{K_Z} \biggl( \frac{1}{d} \biggr)^{K_Z} \biggl( 1 - \frac{1}{d} \biggr)^{T_Z-K_Z} \biggr)
	\\
	& = 2 r \cdot \mathbb{E}_{T_Z} \biggl( 1 - \frac{1}{d} + \frac{3}{4d} \biggr)^{T_Z}
	= 2 r \cdot \mathbb{E}_{T_Z} \biggl( 1 - \frac{1}{4d} \biggr)^{T_Z}.
	\end{align*}
	Observing that when the underlying partition rule $Z$ has number of splits $p$, the partition tree is statistically related to a random binary search tree with $p + 1$ external nodes and $p$ internal nodes. Then, Lemma \ref{SaturationLevel} states that for $k \geq 1$ and $\log(2p+1) > k+\log(k+1)$,
	\begin{align*}
	\mathrm{P}(S_{2p+1} < k+1)
	\leq \biggl( \frac{k+1}{2p+1} \biggr) \biggl( \frac{2e}{k} \log \biggl( \frac{2p+1}{k+1} \biggr) \biggr)^k,
	\end{align*}
	where $S_{2p+1}$ is the \emph{saturation level}. In our setting, $S_{2p+1}$ can be viewed as the maximal number of splits that generates any $A \in \mathcal{A}$. Now taking $k = \lfloor c_T \log(2p+1) \rfloor$ where $c_T < 1$ and $c_T (1 + \log(2e/c_T)) < 1$, simple calculation shows that
	\begin{align*}
	\mathrm{P} ( T_Z < \lfloor c_T \log(2p+1) \rfloor + 1 )
	& \leq \mathrm{P} ( S_{2p+1} < \lfloor c_T \log(2p+1) \rfloor + 1 )
	\\
	& \leq K'' (2p+1)^{c_T(1+\log(2e/c_T))-1}
	\\
	& \leq K'' (2p)^{c_T(1+\log(2e/c_T))-1}
	\\
	& \leq K' p^{c_T(1+\log(2e/c_T))-1},
	\end{align*}
	where $K''$ and $K'$ are universal constants. Consequently for any $A \in \mathcal{A}$, we have
	\begin{align*}
	\mathbb{E}_Z V_Z
	& \leq 2 r \cdot \mathbb{E}_{T_Z} \biggl( \biggl( 1 - \frac{1}{4d} \biggr)^{T_Z} \eins_{\{ T_Z < \lfloor c_T \log(2p+1) \rfloor + 1 \}} \biggr)
	\\
	& \phantom{=}
	+ 2 r \cdot \mathbb{E}_{T_Z} \biggl( \biggl( 1 - \frac{1}{4d} \biggr)^{T_Z} \eins_{\{ T_Z \geq \lfloor c_T \log(2p+1) \rfloor + 1 \}} \biggr)
	\\
	& \leq 2 r K' p^{c_T(1+\log(2e/c_T))-1} 
	+ 2 r \cdot \mathbb{E}_{T_Z} \biggl( \biggl( 1 - \frac{1}{4d} \biggr)^{T_Z} \eins_{\{ T_Z \geq \lfloor c_T \log p \rfloor + 1 \}} \biggr)
	\\
	& \leq K p^{c_T(1+\log(2e/c_T))-1} + 2 r (1 - 1/(4d))^{c_T\log p}
	\\
	& \leq K p^{c_T(1+\log(2e/c_T))-1} + 2 r p^{-c_T/(4d)},
	\end{align*}
	where the last inequality follows from the fact that $1 - 1/x < e^{-x}$ for all $x > 1$. Since the function $f(c_T) = 1 - c_T (1+\log(2e/c_T)) - c_T/(4d)$ is monotone decreasing on $(0, 1)$ for all $d$, numerical computation shows that the largest constant for which $1 - c_T (1+\log(2e/c_T)) > c_T/(4d)$ holds for all $d \geq 1$ cannot be greater than $0.22563$. Therefore, taking $c_T = 0.22$, there holds $\mathbb{E}_Z V_Z \leq (K+2r) p^{-c_T/(4d)}$. Therefore, we obtain that
	\begin{align} \label{ExpectationDiamA}
	\sum_{i=1}^d \mathbb{E}_Z \biggl( \max_{A \in \mathcal{A}_{Z,p}} V_i(A) \biggr) 
	\leq (K + 2r) d p^{-c_T/(4d)}. 
	\end{align}
	According to \eqref{diameterA} and \eqref{ExpectationDiamA}, we have
	\begin{align} \label{ProbDiameterA}
	\mathrm{P}_Z ( \forall A \in \mathcal{A}_{Z,p} : \mathrm{diam}(A) \leq h )
	\geq 1 - (K+2r) d h^{-1} p^{-c_T/(4d)},
	\end{align}
	where $\mathrm{P}_Z$ is the probability measure of partition variable $Z$ defined on $\mathcal{Z}$, $K$ is a universal constant and $c_T = 0.22$ and we set $h = \delta$ and $r = p^{c_T/(8d)}$. According to the requirement for $\delta$ to make \eqref{ftildeUniformlyContinuity} true and the inquality in \eqref{ProbDiameterA}, it is not difficult to see that for any $\varepsilon > 0$ and $\xi > 0$, there exists a $p_{\varepsilon,\xi}$ so that when $p \geq p_{\varepsilon,\xi}$, we have
	\begin{align*}
	\mathrm{P}_Z ( \forall A \in \mathcal{A}_{Z,p} : \mathrm{diam}(A) \leq \delta ) 
	\geq 1 - \xi.
	\end{align*}
	Therefore, for any $x, x' \in A(x)$, we have
	\begin{align*}
	|\tilde{f}(x) - \tilde{f}(x')|
	\leq \frac{\varepsilon}{3 \mu(B_r)}
	\end{align*}
	holds with probability $\mathrm{P}_Z$ at least $1 - \xi$. As a result, we have
	\begin{align} \label{ftildefbarL1}
	\|\tilde{f} - \bar{f}\|_1
	= \int_{B_r} |\tilde{f}(x) - \bar{f}(x)| \, d\mu(x)
	\leq \frac{\varepsilon}{3}
	\end{align}
	holds with probability $\mathrm{P}_Z$ at least $1 - \xi$. Finally, \eqref{PDensity2} and \eqref{fbar} yield that
	\begin{align*}
	\|\bar{f} - f_{\mathrm{P},Z}\|_1
	& = \sum_{\mu(A_j) > 0} \int_{A_j} |\bar{f}(x) - f_{\mathrm{P},Z}(x)| \, d\mu(x)
	\\
	& = \sum_{\mu(A_j) > 0} \int_{A_j} \biggl| \frac{1}{\mu(A_j)} \int_{A_j} \tilde{f}(x') \, d\mu(x')
	- \frac{1}{\mu(A_j)} \int_{A_j} f(x') \, d\mu(x') \biggr| \, d\mu(x)
	\\
	& = \sum_{\mu(A_j) > 0} \biggl| \int_{A_j} \tilde{f}(x') \, d\mu(x') - \int_{A_j} f(x') \, d\mu(x') \biggr|
	\\
	& \leq \sum_{\mu(A_j) > 0} \int_{A_j} |\tilde{f}(x') - f(x')| \, d\mu(x')
	\\
	& \leq \int_{\mathbb{R}^d} |\tilde{f}(x') - f(x')| \, d\mu(x')
	\\
	& \leq \frac{\varepsilon}{3}
	\end{align*}
	and this proves the assertion by combining the estimates in \eqref{L1ErrorDecomposition} and \eqref{ftildefbarL1}.
	
	\emph{(ii)} 
	The combination of $\alpha$-H\"{o}lder continuity of $f$ and \eqref{ProbDiameterA} tells us that for all $x, x' \in A$, we have
	\begin{align*}
	|f(x) - f(x')|
	\leq c \|x - x'\|_1^{\alpha}
	\leq c (\mathrm{diam}(A))^{\alpha}
	\leq c h^{\alpha}
	\end{align*}
	with probability $\mathrm{P}_Z$ at least $1 - (K+2r) d h^{-1} p^{-c_T/(4d)}$. Moreover, for all $x \in B_r$, there exists exactly a number $0 \leq j \leq p$ such that $x \in A_j$. In the following, we write $A(x) := A_j$. Then, the inequality above implies that for $x \in B_r$,
	\begin{align*}
	|f_{\mathrm{P},Z}(x) - f(x)|
	& = \biggl| \frac{1}{\mu(A(x))} \int_{A(x)} f(x') \, d\mu(x') - f(x) \biggr|
	\\
	& = \biggl| \frac{1}{\mu(A(x))} \int_{A(x)} (f(x') - f(x)) \, d\mu(x') \biggr|
	\\
	& \leq \frac{1}{\mu(A(x))} \int_{A(x)} |f(x') - f(x)| \, d\mu(x') \leq c h^{\alpha}
	\end{align*}
	with probability $\mathrm{P}_Z$ at least $1 - (K+2r) d h^{-1} p^{-c_T/(4d)}$. Setting $\xi := (K+2r) d h^{-1} p^{-c_T/(4d)}$, then
	\begin{align*}
	|f_{\mathrm{P},Z}(x) - f(x)|
	\leq c ( (K+2r) d \xi^{-1} p^{-c_T/(4d)} )^{\alpha}
	\end{align*}
	with probability $\mathrm{P}_Z$ at least $1 - \xi$. In the special case when $K \leq r$, we get
	\begin{align*}
	|f_{\mathrm{P},Z}(x) - f(x)|
	\leq c (3 r d \xi^{-1} p^{-c_T/(4d)})^{\alpha}.
	\end{align*}
	It follows that for any $p \geq (3 r d \xi^{-1} (c/\varepsilon)^{1/\alpha})^{4d/c_T}$, there holds
	\begin{align*}
	|f_{\mathrm{P},Z}(x) - f(x)|
	\leq \varepsilon.
	\end{align*}
\end{proof}

\begin{proof}[of Proposition \ref{L1LInftyRelation}] 
	We decompose $\|f_{\mathrm{P},Z} - f\|_1$ as follows
	\begin{align*}
	\|f_{\mathrm{P},Z} - f\|_1
	& = \int_{B_r} |f_{\mathrm{P},Z} - f| \, \mathrm{d}x
	+ \int_{B_r^c} |f_{\mathrm{P},Z} - f| \, \mathrm{d}x
	\\
	& \leq \mu(B_r) \|(f_{\mathrm{P},Z} - f) \eins_{B_r}\|_{\infty}
	+ \int_{B_r^c} f_{\mathrm{P},Z} \, \mathrm{d}x
	+ \int_{B_r^c} f \, \mathrm{d}x
	\\
	& = 2^d r^d \|(f_{\mathrm{P},Z} - f) \eins_{B_r}\|_{\infty}
	+ \int_{B_r^c} f_{\mathrm{P},Z} \, \mathrm{d}x
	+ \mathrm{P}(B_r^c).
	\end{align*}
	Now, \eqref{PDensity1} implies that
	\begin{align*}
	\int_{B_r^c} f_{\mathrm{P},Z} \, \mathrm{d}x
	& = \int_{B_r^c} \sum_{j=0}^p \frac{\mathrm{P}(A_j) \eins_{A_j}(x)}{\mu(A_j)}
	+ \frac{\mathrm{P}(B_r^c) \eins_{B_r^c}(x)}{\mu(B_r^c)} \, d\mu(x)
	\\
	& = \frac{\mathrm{P}(B_r^c)}{\mu(B_r^c)} \int_{B_r^c} \eins_{B_r^c}(x) \, d\mu(x)
	= \frac{\mathrm{P}(B_r^c)}{\mu(B_r^c)} \mu(B_r^c)
	= \mathrm{P}(B_r^c).
	\end{align*}
	Combining the above two estimates, we obtain the desired conclusion. 
\end{proof}

\begin{proof}[of Lemma \ref{FundamentalLemma}] 
	\emph{(i)} 
	By definition, we have
	\begin{align*}
	\|f_{\mathrm{D},Z} - f_{\mathrm{P},Z}\|_1
	& = \int_{\mathbb{R}^d} |f_{\mathrm{D},Z} - f_{\mathrm{P},Z}| \, d\mu
	\\
	& = \int_{\mathbb{R}^d} \biggl| \sum_{j=0}^p \frac{1}{\mu(A_j)} (\mathrm{D}(A_j) - \mathrm{P}(A_j)) \eins_{A_j}
	+ \frac{1}{\mu(B_r^c)} (\mathrm{D}(B_r^c)-\mathrm{P}(B_r^c)) \eins_{B_r^c} \biggr| \, d\mu
	\\
	& = \sum_{j=0}^p \int_{A_j} \frac{1}{\mu(A_j)} |\mathrm{D}(A_j) - \mathrm{P}(A_j)| \, d\mu
	+ \int_{B_r^c} \frac{1}{\mu(B_r^c)} |\mathrm{D}(B_r^c) - \mathrm{P}(B_r^c)| \, d\mu
	\\
	& = \sum_{j=0}^p |\mathrm{D}(A_j) - \mathrm{P}(A_j)| + |\mathrm{D}(B_r^c) - \mathrm{P}(B_r^c)|
	\\
	& = \sum_{j=0}^p |\mathbb{E}_{\mathrm{D}} \eins_{A_j} - \mathbb{E}_{\mathrm{P}} \eins_{A_j}|
	+ |\mathbb{E}_{\mathrm{D}} \eins_{B_r^c} - \mathbb{E}_{\mathrm{P}} \eins_{B_r^c}|.
	\end{align*}
	
	\emph{(ii)} 
	Denote $A_{p+1} := B_r^c$. Using Equalitiy \eqref{PDensity1} and Equalitiy \eqref{RandomDensityTree}, we get
	\begin{align*}
	\|f_{\mathrm{D},Z} - f_{\mathrm{P},Z}\|_{\infty}
	& = \sup_{j \in \{ 0, \ldots, p+1 \}} \sup_{x \in A_j} |f_{\mathrm{D},Z}(x) - f_{\mathrm{P},Z}(x)|
	\\
	& = \sup_{j \in \{ 0, \ldots, p+1 \}} \sup_{x \in A_j} \biggl| \frac{\mathrm{D}(A_j)}{\mu(A_j)} - \frac{\mathrm{P}(A_j)}{\mu(A_j)}|
	= \sup_{j \in \{ 0, \ldots, p+1 \}} \frac{1}{\mu(A_j)} |\mathrm{D}(A_j) - \mathrm{P}(A_j)|.
	\end{align*}
\end{proof}

\begin{proof}[of Proposition \ref{OracleInequalityL1}] 
	By definition, we have
	\begin{align*}
	\|f_{\mathrm{D}, Z} - f_{\mathrm{P}, Z} \|_1
			& = \sum_{j=0}^p |\mathbb{E}_{\mathrm{D}} \eins_{A_j} - \mathbb{E}_{\mathrm{P}} \eins_{A_j}| 
			+ |\mathbb{E}_{\mathrm{D}} \eins_{B_r^c} - \mathbb{E}_{\mathrm{P}} \eins_{B_r^c}|
			\\
			& \leq \sup_{\pi_p \in \pi_{(p)}} \sum_{A \in \pi_p} |\mathbb{E}_{\mathrm{D}} \eins_A - \mathbb{E}_{\mathrm{P}} \eins_A|
			+ |\mathbb{E}_{\mathrm{D}} \eins_{B_r^c} - \mathbb{E}_{\mathrm{P}} \eins_{B_r^c}|,
			\end{align*}
			where $\pi_p$ is a partition of $B_r$ with number of splits $p$ and $\pi_{(p)}$ denote the collection of all partitions $\pi_p$. Here, we define
			\begin{align*}
			\mathcal{B}(\pi_p)
			:= \biggl\{ B : B = \bigcup_{j \in J} A_j, J \subset \{ 0, 1, \ldots, p \}, A_j \in \pi_p \biggr\}
			\end{align*}
			as the collection of all $2^{p+1}$ sets that can be expressed as the union of cells of $\pi_p$. Also, \eqref{Bp} denotes the collection of all such unions, as $\pi_p$ ranges through $\pi_{(p)}$. Fix $\pi_p$ for the moment and define
			\begin{align*}
			\tilde{A} = \bigcup_{A \in \pi_p : \mathbb{E}_{\mathrm{D}} \eins_A \geq \mathbb{E}_{\mathrm{P}} \eins_A} A.
			\end{align*}
			Then clearly
			\begin{align*}
			\sum_{A \in \pi_p} |\mathbb{E}_{\mathrm{D}} \eins_A - \mathbb{E}_{\mathrm{P}} \eins_A|
			= 2 ( \mathbb{E}_{\mathrm{D}} \eins_{\tilde{A}} - \mathbb{E}_{\mathrm{P}} \eins_{\tilde{A}})
			\leq 2 \sup_{B \in \mathcal{B}(\pi_p)} |\mathbb{E}_{\mathrm{D}} \eins_B - \mathbb{E}_{\mathrm{P}} \eins_B|.
			\end{align*}
			Consequently,
			\begin{align}
			\|f_{\mathrm{D},Z} - f_{\mathrm{P},Z}\|_1
			& \leq \sup_{\pi_p \in \pi_{(p)}} \sum_{A \in \pi_p} |\mathbb{E}_{\mathrm{D}} \eins_A - \mathbb{E}_{\mathrm{P}} \eins_A|
			+ |\mathbb{E}_{\mathrm{D}} \eins_{B_r^c} - \mathbb{E}_{\mathrm{P}} \eins_{B_r^c}|
			\nonumber\\
			& \leq 2 \sup{\pi_p \in \pi_{(p)}} \sup_{B \in \mathcal{B}(\pi_p)} |\mathbb{E}_{\mathrm{D}} \eins_B - \mathbb{E}_{\mathrm{P}} \eins_B| \
			+ |\mathbb{E}_{\mathrm{D}} \eins_{B_r^c} - \mathbb{E}_{\mathrm{P}} \eins_{B_r^c}|
			\nonumber\\
			& = 2 \sup_{B \in \mathcal{B}_p} |\mathbb{E}_{\mathrm{D}} \eins_B - \mathbb{E}_{\mathrm{P}} \eins_B|
			+ |\mathbb{E}_{\mathrm{D}} \eins_{B_r^c} - \mathbb{E}_{\mathrm{P}} \eins_{B_r^c}|.
			\label{EstErrorDecomp}
			\end{align}
			To bound the $\|f_{\mathrm{D},Z} - f_{\mathrm{P},Z}\|_1$, we first consider bounding the first term on the right side of \eqref{EstErrorDecomp}. For a fix $B \in \mathcal{B}_p$, we consider the map
			\begin{align*}
			\xi_i := \eins_B \pi_i - \mathbb{E}_{\mathrm{P}} \eins_B,
			\end{align*}
			where $\pi_i$ is the $i$-th projection, that is $\pi_i(\{ X_i \}_{i=1}^n) := X_i$. Then, we verify the following conditions: Obviously, we have $\mathbb{E}_{\mathrm{P}^n} \xi_i = 0$ and $\|\xi_i\|_{\infty} \leq 1$. Moreover, simple estimates imply
			\begin{align}
			\mathbb{E}_{\mathrm{P}^n} \xi_i^2
			\leq \mathbb{E}_{\mathrm{P}} \eins_B^2 - (\mathbb{E}_{\mathrm{P}} \eins_B)^2
			& = \mathbb{E}_{\mathrm{P}} \eins_B - (\mathbb{E}_{\mathrm{P}} \eins_B)^2
			\nonumber\\
			& = \mathrm{P}(B) - \mathrm{P}^2(B)
			= \mathrm{P}(B) (1 - \mathrm{P}(B))
			\leq \frac{1}{4}. 
			\label{VarianceBound}
			\end{align}
			Finally, it is easy to see that $(\xi_i)$ are independent with respect to $\mathrm{P}^n$. Therefore, we can apply Bernstein's inequality and obtain that for all $n \geq 1$, with probability $\mathrm{P}^n$ at most $2 e^{-\tau}$, there holds
			\begin{align} \label{BernEstimate1}
			|\mathbb{E}_{\mathrm{D}} \eins_B - \mathbb{E}_{\mathrm{P}} \eins_B|
			= \biggl| \frac{1}{n} \sum_{i=1}^n \xi_i \biggr|
			\geq \sqrt{\frac{\tau}{2n}} + \frac{2\tau}{3n}.
			\end{align}
			We choose $B_1, \ldots, B_m \in \mathcal{B}_p$ such that $\{ B_1, \ldots, B_m \}$ is an $\varepsilon$-net of $\mathcal{B}_p$ with respect to $\|\cdot\|_{L_1(\mathrm{D})}$. Note that here we have $m = \mathcal{N} (\eins_{\mathcal{B}_p}, \|\cdot\|_{L_1(\mathrm{D})}, \varepsilon)$ as in Lemma \ref{BpCoveringNumber}. From the estimate in \eqref{BernEstimate1} and a union bound argument, with probability $\mathrm{P}^n$ at least $1 - 2 m e^{-\tau}$, the following estimate holds
			\begin{align} \label{BernEstimateMax}
			\sup_{j = 1, \ldots, m} |\mathbb{E}_{\mathrm{D}} \eins_{B_j} - \mathbb{E}_{\mathrm{P}} \eins_{B_j}|
			\leq \sqrt{\frac{\tau}{2n}} + \frac{2\tau}{3n}.
			\end{align}
			Recalling that $\{ B_1, \ldots, B_m \}$ is an $\varepsilon$-net of $\mathcal{B}_p$, this implies that, for any $B \in \mathcal{B}_p$, there exists $B_j$ such that
			\begin{align*}
			\bigl| |\mathbb{E}_{\mathrm{D}} \eins_B - \mathbb{E}_{\mathrm{P}} \eins_B| 
			- |\mathbb{E}_{\mathrm{D}} \eins_{B_j} - \mathbb{E}_{\mathrm{P}} \eins_{B_j}| \bigr|
			& \leq |\mathbb{E}_{\mathrm{D}} \eins_B - \mathbb{E}_{\mathrm{P}} \eins_B 
			- (\mathbb{E}_{\mathrm{D}} \eins_{B_j} - \mathbb{E}_{\mathrm{P}} \eins_{B_j})|
			\\
			& \leq |\mathbb{E}_{\mathrm{D}} \eins_B - \mathbb{E}_{\mathrm{D}} \eins_{B_j}|
			+ |\mathbb{E}_{\mathrm{P}} \eins_B - \mathbb{E}_{\mathrm{P}} \eins_{B_j}|.
			\end{align*}
			Now, we calculate two terms on the right hand of the above inequality seperately. For the first term,
			\begin{align*}
			|\mathbb{E}_{\mathrm{D}} \eins_B - \mathbb{E}_{\mathrm{D}} \eins_{B_j}|
			& = \biggl| \frac{1}{n} \sum_{i=1}^n (\eins_B \pi_i - \eins_{B_j} \pi_i) \biggr|
			\\
			& \leq \frac{1}{n} \sum_{i=1}^n |\eins_B \pi_i - \eins_{B_j} \pi_i|
			= \|\eins_B - \eins_{B_j}\|_{L_1(\mathrm{D})}
			\leq \varepsilon.
			\end{align*}
			As for the second term,
			\begin{align*}
			|\mathbb{E}_{\mathrm{P}} \eins_B - \mathbb{E}_{\mathrm{P}} \eins_{B_j}|
			& = \biggl| \int_{\mathbb{R}^d} f \cdot (\eins_B - \eins_{B_j}) \, d\mu \biggr|
			\\
			& \leq \int_{\mathbb{R}^d} f \cdot |\eins_B - \eins_{B_j}| \, d\mu
			= \|\eins_B - \eins_{B_j}\|_{L_1(\mathrm{P})}
			= \mathbb{E}_{\mathrm{P}} (\|\eins_B - \eins_{B_j}\|_{L_1(\mathrm{D})})
			\leq \varepsilon.
			\end{align*}
			Consequently, we have
			\begin{align*}
			|\mathbb{E}_{\mathrm{D}} \eins_B - \mathbb{E}_{\mathrm{P}} \eins_B|
			\leq |\mathbb{E}_{\mathrm{D}} \eins_{B_j} - \mathbb{E}_{\mathrm{P}} \eins_{B_j}| + 2 \varepsilon.
			\end{align*}
			This together with \eqref{BernEstimateMax} implies that for any $B \in \mathcal{B}_p$, there holds
			\begin{align} \label{BernEstimate2}
			|\mathbb{E}_{\mathrm{D}} \eins_B - \mathbb{E}_{\mathrm{P}} \eins_B|
			\leq \sqrt{\frac{\tau}{2n}} + \frac{2\tau}{3n} + 2 \varepsilon
			\end{align}
			with probability $\mathrm{P}^n$ at least $1 - 2 m e^{-\tau}$.
			
			Next, we focus on bounding the second term on the right side of \eqref{EstErrorDecomp}. For a fixed $r > 0$, we consider the map
			\begin{align*}
			\xi_i := \eins_{B_r^c} \pi_i - \mathbb{E}_{\mathrm{P}} \eins_{B_r^c},
			\end{align*}
			where $\pi_i$ is the $i$-th projection. Then, we verify the following conditions: Obviously, we have $\mathbb{E}_{\mathrm{P}^n} \xi_i = 0$ and $\|\xi_i\|_{\infty} \leq 1$. Moreover, simliar estimates as \eqref{VarianceBound} imply
			\begin{align*}
			\mathbb{E}_{\mathrm{P}^n} \xi_i^2
			\leq \mathrm{P}(B_r^c) (1 - \mathrm{P}(B_r^c))
			\leq \frac{1}{4}.
			\end{align*}
			Also, we can see that $(\xi_i)$ are independent with respect to $\mathrm{P}^n$. Therefore, we can apply Bernstein's inequality once again and obtain that for all $n \geq 1$, with probability $\mathrm{P}^n$ at most $2 e^{-\tau}$, there holds
			\begin{align} \label{BernEstimateComplement}
			|\mathbb{E}_{\mathrm{D}} \eins_{B_r^c} - \mathbb{E}_{\mathrm{P}} \eins_{B_r^c}|
			\geq \sqrt{\frac{\tau}{2n}} + \frac{2\tau}{3n}.
			\end{align}
			Combining \eqref{BernEstimate2} and \eqref{BernEstimateComplement} yields
			\begin{align*}
			\|f_{\mathrm{D},Z} - f_{\mathrm{P},Z}\|_1
			\leq \sqrt{\frac{9\tau}{2n}} + \frac{2\tau}{n} + 4 \varepsilon
			\end{align*}
			with probability $\mathrm{P}^n$ at least $1 - 2 (m+1) e^{-\tau}$. By a simple variable transformation, we see that with probability $\mathrm{P}^n$ at least $1 - e^{-\tau}$, there holds
			\begin{align} \label{EstimationErrorInbetween}
			\|f_{\mathrm{D},Z} - f_{\mathrm{P},Z}\|_1
			\leq \sqrt{\frac{9 (\tau + \log(2 m + 2))}{2n}} + \frac{2(\tau+\log(2m+2))}{n} + 4 \varepsilon.
			\end{align}
			Next, we calculate the term $\log(2m+2)$. With $m = \mathcal{N}(\eins_{\mathcal{B}_p}, \|\cdot\|_{L_1(\mathrm{D})}, \varepsilon)$ in Lemma \ref{BpCoveringNumber}, for any $1/\varepsilon > \max \{ e, 2K+2 \}$, $p \geq 1$, we derive
			\begin{align}
			\log(2m+2)
			& \leq \log ( 2K (dp+2) (4e)^{dp+2} (1/\varepsilon)^{dp+1} + 2 )
			\nonumber\\
			& \leq \log ( (2K+2) (dp+2) (4e)^{dp+2} (1/\varepsilon)^{dp+1} )
			\nonumber\\
			& = \log (2K+2) + \log (dp+2) + (dp+2) \log (4e) + (dp+1) \log (1/\varepsilon)
			\nonumber\\
			& \leq 15 d p \log(1/\varepsilon),  
			\label{ElemEstimate}
			\end{align}
			where the last inequality is based on the following four basic inequalities: 
			$\log (2K+2) \leq \log (1/\varepsilon) \leq d p \log (1/\varepsilon)$, 
			$\log (dp+2) \leq dp+2 \leq 3 d p \leq 3 d p \log(1/\varepsilon)$, 
			$(dp+2) \log (4e) \leq 3 d p \log(e^3) \leq 9 d p \log(1/\varepsilon)$, 
			$(dp+1) \log (1/\varepsilon) \leq 2 d p \log(1/\varepsilon)$. 
			Now, when choosing $\varepsilon = 1/n$ and plugging \eqref{ElemEstimate} into \eqref{EstimationErrorInbetween}, we obtain
			\begin{align*}
			\|f_{\mathrm{D},Z} - f_{\mathrm{P},Z}\|_1
			\leq \sqrt{\frac{9(\tau+15dp\log n)}{2n}} + \frac{2(\tau+15dp\log n)}{n} + \frac{4}{n}
			\end{align*}
			with probability $\mathrm{P}^n$ at least $1 - e^{-\tau}$.
		\end{proof}

To prove Proposition \ref{OracleInequalityLInfty}, we need the following two Lemmas where Lemma 17 is a result from \cite{devroye1986a}.

\begin{lemma} \label{Height}
	Let $H_n$ be the height of a binary search tree with $n$ nodes. Then for any integer $k \geq \max(1, \log n)$, we have
	\begin{align*}
	\mathrm{P}(H_n \geq k)
	\leq \frac{1}{n} \biggl( \frac{2e\log n}{k} \biggr)^k.
	\end{align*}
\end{lemma}

\begin{lemma} \label{CollectionCoveringNumber}
	Let $\tilde{\mathcal{A}}$ be the collection of all cells $\bigtimes_{i=1}^d [a_i, b_i]$ in $\mathbb{R}^d$. The VC index of $\tilde{\mathcal{A}}$ equals $2d+1$. Moreover, for all $0 < \varepsilon < 1$, there exists a universal constant $K$ such that
	\begin{align*}
	\mathcal{N}(\eins_{\tilde{\mathcal{A}}}, \|\cdot\|_{L_1(Q)}, \varepsilon)
	\leq K (2d+1) (4e)^{2d+1} \biggl( \frac{1}{\varepsilon} \biggr)^{2d}
	\end{align*}
	holds for any probability measure $Q$.
\end{lemma}

\begin{proof}[of Lemma \ref{CollectionCoveringNumber}] 
	The first result of VC index follows from Example 2.6.1 in \cite{vandervaart1996weak}. The second result of covering number follows directly from Theorem 9.2 in \cite{Kosorok2008introduction}.
\end{proof}

\begin{proof}[of Proposition \ref{OracleInequalityLInfty}] 
	Since the density function $f$ considered has a bounded support, we set $r$ large enough so that the entire support can be contained in the $B_r$. According to Lemma \ref{FundamentalLemma}, we have
	\begin{align*}
	\|f_{\mathrm{D},Z} - f_{\mathrm{P},Z}\|_{\infty}
	= \sup_{j \in \{ 0, \ldots, p \}} \frac{1}{\mu(A_j)} |\mathbb{E}_{\mathrm{D}} \eins_{A_j} - \mathbb{E}_{\mathrm{P}} \eins_{A_j}|.
	\end{align*}
	Considering that the partitions are conducted at random, in order to bound the above estimation error in $L_{\infty}$ norm, we carry out the following proof in two part. First of all, the bounding work will be conducted under the condition that $\mu(A) \geq h$ for all cells. Secondly, we give analysis on the probability of $\mu(A) \geq h$ for all cells in random partitions.
	
	We first introduce a space defined by
	\begin{align*}
	\tilde{\mathcal{A}} := \biggl\{ A := \bigtimes_{i=1}^d [a_i, b_i] : [a_i, b_i] \subset [-r, r], i = 1, \ldots, d \text{ and } \mu(A) \geq h \biggr\}.
	\end{align*}
	Then, under the condition that $\mu(A_j) \geq h$, $j = 0, \ldots, p$, we obtain that
	\begin{align*}
	\sup_{j \in \{ 0, \ldots, p \}} \frac{1}{\mu(A_j)} |\mathbb{E}_{\mathrm{D}} \eins_{A_j} - \mathbb{E}_{\mathrm{P}} \eins_{A_j}|
	\leq \sup_{A \in \tilde{\mathcal{A}}} \frac{1}{\mu(A)} |\mathbb{E}_{\mathrm{D}} \eins_A - \mathbb{E}_{\mathrm{P}} \eins_A|.
	\end{align*}
	For a fixed $A \in \tilde{\mathcal{A}}$, we estimate
	\begin{align*}
	\frac{1}{\mu(A)} |\mathbb{E}_{\mathrm{D}} \eins_A - \mathbb{E}_{\mathrm{P}} \eins_A|
	\end{align*}
	by using Bernstein's inequality. For this purpose, we consider the map
	\begin{align*}
	\xi_i := \frac{1}{\mu(A)} (\eins_A \pi_i - \mathbb{E}_{\mathrm{P}} \eins_A),
	\end{align*}
	where $\pi_i$ is the $i$-th projection. Then, we verify the following conditions: Obviously, we have $\mathbb{E}_{\mathrm{P}^n} \xi_i = 0$ and $\|\xi_i\|_{\infty} \leq \frac{1}{\mu(A)} \leq \frac{1}{h}$. Moreover, simple estimates imply
	\begin{align*}
	\mathbb{E}_{\mathrm{P}^n} \xi_i^2 
	& \leq \frac{1}{\mu^2(A)} \mathbb{E}_{\mathrm{P}} \eins_A^2
	= \frac{1}{\mu^2(A)} \mathbb{E}_{\mathrm{P}} \eins_A
	= \frac{1}{\mu^2(A)} \mathrm{P}(A)
	= \frac{1}{\mu^2(A)} \int_A f(x) \, d\mu(x)
	\\
	& \leq \frac{1}{\mu^2(A)} \|f\|_{\infty} \int_A 1 \, d\mu(x)
	= \frac{1}{\mu(A)} \|f\|_{\infty}
	\leq \frac{\|f\|_{\infty}}{h}.
	\end{align*}
	Finally, it is easy to see that $(\xi_i)$ are independent with respect to $\mathrm{P}^n$. Therefore, we can apply Bernstein's inequality and obtain that for all $n \geq 1$, with probability at most $2 e^{-\tau}$, there holds
	\begin{align} \label{BernEstimateInfty1}
	\frac{1}{\mu(A)} |\mathbb{E}_{\mathrm{D}} \eins_A - \mathbb{E}_{\mathrm{P}} \eins_A|
	\geq \sqrt{\frac{2 \|f\|_{\infty} \tau}{n h}} + \frac{2\tau}{3nh}.
	\end{align}
	We choose $A_1, \ldots, A_m \in \tilde{\mathcal{A}}$ such that $\{ A_1, \ldots, A_m \}$ is an $\tilde{\varepsilon}$-net of $\tilde{\mathcal{A}}$ with respect to $\|\cdot\|_{L_1(Q)}$ where we set $\tilde{\varepsilon} := \frac{h^2}{h+\mu(B_r)} \varepsilon$ and $Q$ is any probability measure. In this case, we have $m = \mathcal{N}(\tilde{\mathcal{A}}, \|\cdot\|_{L_1(Q)}, \tilde{\varepsilon})$. From the estimate in \eqref{BernEstimateInfty1} and a union bound argument, with probability $\mathrm{P}^n$ at least $1 - 2 m e^{-\tau}$, the following estimate holds
	\begin{align} \label{BernEstimateInftyMax}
	\sup_{j \in \{ 1, \ldots, m \}} \frac{1}{\mu(A_j)} |\mathbb{E}_{\mathrm{D}} \eins_{A_j} - \mathbb{E}_{\mathrm{P}} \eins_{A_j}|
	\leq \sqrt{\frac{2\|f\|_{\infty} \tau}{nh}} + \frac{2\tau}{3nh}.
	\end{align}
	Recalling that $\{ A_1, \ldots, A_m \}$ is an $\tilde{\varepsilon}$-net of $\tilde{\mathcal{A}}$ with respect to any probability measure $Q$, this implies that, for any $A \in \tilde{\mathcal{A}}$, there exists $A_j$ such that
	\begin{align*}
	\|\eins_A - \eins_{A_j}\|_{L_1(Q)} \leq \tilde{\varepsilon}.
	\end{align*}
	Therefore, for $Q = \mathrm{D}$, we have
	\begin{align} 
	\biggl\| \frac{1}{\mu(A)} \eins_A - \frac{1}{\mu(A_j)} \eins_{A_j} \biggr\|_{L_1(\mathrm{D})}
	& = \biggl\| \frac{1}{\mu(A)} \eins_A - \frac{1}{\mu(A_j)} \eins_A + \frac{1}{\mu(A_j)} \eins_A - \frac{1}{\mu(A_j)} \eins_{A_j} \biggr\|_{L_1(\mathrm{D})}
	\nonumber\\
	& \leq \frac{|\mu(A_j) - \mu(A)|}{\mu(A) \mu(A_j)} \|\eins_A\|_{L_1(\mathrm{D})}
	+ \frac{1}{\mu(A_j)} \|\eins_A - \eins_{A_j}\|_{L_1(\mathrm{D})},
	\label{EstimateMuAAj}
	\end{align}
	where $\mu(A) \geq h$ and $\mu(A_j) \geq h$. Now, we bound two terms on the right hand of the above inequality seperately. For the second term, it can be apparently seen that
	\begin{align} \label{EstimateMuAAj1}
	\frac{1}{\mu(A_j)} \|\eins_A - \eins_{A_j}\|_{L_1(\mathrm{D})}
	\leq \frac{1}{\mu(A_j)} \tilde{\varepsilon}
	\leq \frac{\tilde{\varepsilon}}{h}.
	\end{align}
	On the other hand, if $Q$ is the uniform distribution on $B_r$, then we have
	\begin{align*}
	\|\eins_A - \eins_{A_j}\|_{L_1(\mathrm{Q})}
	= \int |\eins_A - \eins_{A_j}| \frac{1}{\mu(B_r)}\, d\mu
	= \int \frac{1}{\mu(B_r)} \eins_{A \triangle A_j} \, d\mu
	= \frac{\mu(A \triangle A_j)}{\mu(B_r)}
	\leq \tilde{\varepsilon}
	\end{align*}
	and therefore,
	\begin{align*}
	\max \{ h, \mu(A) - \tilde{\varepsilon} \mu(B_r) \} 
	& \leq \max \{ h, \mu(A) - \mu(A \triangle A_j) \}
	\\
	& \leq \mu(A_j) 
	\leq \mu(A) + \mu(A \triangle A_j)
	\leq \mu(A) + \tilde{\varepsilon} \mu(B_r).
	\end{align*}
	Consequently, we obtain
	\begin{align} \label{EstimateMuAAj2}
	|\mu(A) - \mu(A_j)| \leq \tilde{\varepsilon} \mu(B_r).
	\end{align}
	Now, combining \eqref{EstimateMuAAj2} with \eqref{EstimateMuAAj1}, we can bound \eqref{EstimateMuAAj} by
	\begin{align} \label{EstimateMuAAj3}
	\|\frac{1}{\mu(A)} \eins_A - \frac{1}{\mu(A_j)} \eins_{A_j}\|_{L_1(\mathrm{D})}
	\leq \frac{h + \mu(B_r)}{h^2} \tilde{\varepsilon}
	= \varepsilon.
	\end{align}
	Similarly, it can be deduced that
	\begin{align} \label{EstimateMuAAjP}
	\biggl\| \frac{1}{\mu(A)} \eins_A - \frac{1}{\mu(A_j)} \eins_{A_j} \biggr\|_{L_1(\mathrm{P})} \leq \varepsilon.
	\end{align}
	Then, \eqref{EstimateMuAAj3} and \eqref{EstimateMuAAjP} imply that for any $A \in \tilde{\mathcal{A}}$, there holds
	\begin{align*}
	& \biggl| \frac{1}{\mu(A)} |\mathbb{E}_{\mathrm{D}} \eins_A - \mathbb{E}_{\mathrm{P}} \eins_A| 
	- \frac{1}{\mu(A_j)} |\mathbb{E}_{\mathrm{D}} \eins_{A_j} - \mathbb{E}_{\mathrm{P}} \eins_{A_j}| \biggr|
	\\
	& \leq \biggl| \frac{1}{\mu(A)} (\mathbb{E}_{\mathrm{D}} \eins_A - \mathbb{E}_{\mathrm{P}} \eins_A) 
	- \frac{1}{\mu(A_j)} (\mathbb{E}_{\mathrm{D}} \eins_{A_j} - \mathbb{E}_{\mathrm{P}} \eins_{A_j}) \biggr|
	\\
	& \leq \biggl| \frac{1}{\mu(A)} \mathbb{E}_{\mathrm{D}} \eins_A - \frac{1}{\mu(A_j)} \mathbb{E}_{\mathrm{D}} \eins_{A_j} \biggr|
	+ \biggl| \frac{1}{\mu(A)} \mathbb{E}_{\mathrm{P}} \eins_A - \frac{1}{\mu(A_j)} \mathbb{E}_{\mathrm{P}} \eins_{A_j} \biggr|
	\\
	& \leq \biggl\| \frac{1}{\mu(A)} \eins_A - \frac{1}{\mu(A_j)} \eins_{A_j} \biggr\|_{L_1(\mathrm{D})}
	+ \biggl\| \frac{1}{\mu(A)} \eins_A - \frac{1}{\mu(A_j)} \eins_{A_j} \biggr\|_{L_1(\mathrm{P})}
	\\
	& \leq 2 \varepsilon
	\end{align*}
	and consequently we have
	\begin{align*}
	\frac{1}{\mu(A)} |\mathbb{E}_{\mathrm{D}} \eins_A - \mathbb{E}_{\mathrm{P}} \eins_A|
	\leq \frac{1}{\mu(A_j)} |\mathbb{E}_{\mathrm{D}} \eins_{A_j} - \mathbb{E}_{\mathrm{P}} \eins_{A_j}| + 2 \varepsilon.
	\end{align*}
	This together with \eqref{BernEstimateInftyMax} implies that for any $A \in \tilde{\mathcal{A}}$, there holds
	\begin{align} \label{EstimateInbetween2}
	\frac{1}{\mu(A)} |\mathbb{E}_{\mathrm{D}} \eins_A - \mathbb{E}_{\mathrm{P}} \eins_A|
	\leq \sqrt{\frac{2 \|f\|_{\infty} \tau}{nh}} + \frac{2\tau}{3nh} + 2 \varepsilon
	\end{align}
	with probability $\mathrm{P}^n$ at least $1 - 2  m e^{-\tau}$. By a simple variable transformation, we see that with probability $\mathrm{P}^n$ at least $1 - e^{-\tau}$, there holds under the condition that $\mu(A_j) \geq h$, $j = 0, \ldots, p$
	\begin{align} \label{EstimErrorInfty}
	\|f_{\mathrm{D},Z} - f_{\mathrm{P},Z}\|_{\infty}
	\leq \sqrt{\frac{2 \|f\|_{\infty} (\tau + \log(2m))}{nh}}
	+ \frac{2(\tau+\log(2m))}{3nh} + 2 \varepsilon.
	\end{align}
	Next, we estimate the term $\log(2m)$. With $m := \mathcal{N}(\tilde{\mathcal{A}}, \|\cdot\|_{L_1(Q)}, \tilde{\varepsilon})$ in Lemma \ref{CollectionCoveringNumber}, for any $\varepsilon \in (0, 1/\max \{ e, 2K, \mu(B_r) \})$, there holds
	\begin{align}
	\log(2m)
	& \leq \log (2K (2d+1) (4e)^{2d+1} (1/\varepsilon)^{2d})
	\nonumber\\
	& \leq \log (2K (2d+1) (4e)^{2d+1} (\mu(B_r)+h)^{2d} (1/(\varepsilon h^2))^{2d})
	\nonumber\\
	& = \log (2K) + \log (2d+1) + (2d+1) \log(4e) + 2 d \log (\mu(B_r)+h)
	\nonumber\\
	& \phantom{=}
	+ 4 d \log(1/h) + 2 d \log(1/\varepsilon)
	\nonumber\\
	& \leq 19 d \log(1/\varepsilon) + 4 d \log(1/h),
	\label{EstimateInbetween3}
	\end{align}
	where the last inequality is based on the following four basic inequalities: 
	$\log(2K) \leq \log(1/\varepsilon) \leq d \log(1/\varepsilon)$, 
	$\log(2d+1) \leq 2d+1 \leq 3d \leq 3d \log(1/\varepsilon)$, 
	$(2d+1) \log(4e) \leq 3d \log(e^3) \leq 9d \log(1/\varepsilon)$, 
	$2d \log(\mu(B_r)+h) \leq 2d \log(2\mu(B_r)) \leq 4d \log(1/\varepsilon)$. 
	Now, when choosing $\varepsilon = 1/n$ and plugging \eqref{EstimateInbetween3} into \eqref{EstimErrorInfty}, we obtain under the condition that $\mu(A_j) \geq h$ for all $j = 0, \ldots, p$
	\begin{align}
	\|f_{\mathrm{D},Z} - f_{\mathrm{P},Z}\|_{\infty}
	& \leq \sqrt{\frac{2 \|f\|_{\infty} (\tau + 19 d \log n + 4 d \log(1/h))}{nh}}
	\nonumber\\
	& \phantom{=}
	+ \frac{2(\tau+19d\log n+4d\log(1/h))}{3nh} + \frac{2}{n}
	\label{EstimErrorInfty2}
	\end{align}
	with probability $\mathrm{P}^n$ at least $1 - e^{-\tau}$.
	
	Now, we give analysis on the probability of $\mu(A) \geq h$ for all cells in random partitions. Recall that the random $p$-split partition $\mathcal{A}_{Z,p} = (A_j)_{j=0}^p$ generated by $Z$ is called a random partition. For all $A \in \mathcal{A}_{Z,p}$, the number of splits that generate $A$ is denoted by $T_A$. Then there holds for all $A \in \mathcal{A}_{Z,p}$ and any $k \in \{ 0, \ldots, p \}$ that
	\begin{align} \label{ProbMuAGh}
	\mathrm{P}_Z (\mu(A) \geq h)
	\geq \mathrm{P}_Z (\mu(A) \geq h | T_A < k+1) \mathrm{P}_Z(T_A < k+1).
	\end{align}
	Since our random partition is statistically related to a binary search tree with $p + 1$ externel nodes and $p$ internel nodes, the height of the binary search tree can be viewed as the maximal number of splits that generate $A \in \mathcal{A}_{Z,p}$, i.e. $H_{2p+1} := \max_{A \in \mathcal{A}_{Z,p}} T_A$, and
	\begin{align} 
	\mathrm{P}_Z (T_A < k+1)
	\geq \mathrm{P}_Z (H_{2p+1} < k+1)
	& = 1 - \mathrm{P}_Z (H_{2p+1} \geq k+1)
	\nonumber\\
	& \geq 1 - \frac{1}{2p+1} \biggl( \frac{2 e \log(2p+1)}{k+1} \biggr)^{k+1}
	\label{ProbTAL}
	\end{align}
	holds for any integer $k + 1 \geq \max (1, \log(2p+1))$ and we mention that the last inequality in \eqref{ProbTAL} is a direct result of Lemma \ref{Height}.

	Let $U_1, U_2, \ldots, U_k$ be independent and identically distributed $\mathcal{U}(0, 1)$ random variables, then $V_i := 2 \min (U_i, 1 - U_i)$ is distributed as $\mathcal{U}(0, 1)$ as well. Then it follows that
	\begin{align*}
	\mathrm{P}_Z (\mu(A) \geq h | T_A < k+1)
	\geq \mathrm{P} \biggl( \mu(B_r) \prod_{i=1}^k \min(U_i, 1 - U_i) \geq h \biggr)
	= \mathrm{P} \biggl( \prod_{i=1}^k V_i \geq \frac{2^k h}{\mu(B_r)} \biggr).
	\end{align*}
	Since $V_i \sim \mathcal{U}(0, 1)$, simple tansformation in probability distribution shows that $- \log V_i \sim \exp(1)$, which further leads to $Y := - \sum_{i=1}^k \log V_i \sim \Gamma(k, 1)$. From the above relationships, we obtain
	\begin{align} \label{ProbProdVi}
	\mathrm{P} \biggl( \prod_{i=1}^k V_i \geq \frac{2^k h}{\mu(B_r)} \biggr)
	= \mathrm{P} \biggl( e^{-Y} \geq \frac{2^k h}{\mu(B_r)} \biggr)
	\geq 1 - \frac{h}{\mu(B_r)} \biggl( \frac{2 e \log(\mu(B_r)/(2^kh))}{k} \biggr)^k,
	\end{align}
	where the inequality follows from the Chernoff bound argument which holds when $k < \log (\mu(B_r)/(2^kh))$. Combining \eqref{ProbTAL} with \eqref{ProbProdVi}, we obtain \eqref{ProbMuAGh} as
	\begin{align*}
	\mathrm{P}_Z ( \mu(A) \geq h )
	\geq 1 - \frac{1}{2p+1} \biggl( \frac{2e\log(2p+1)}{k+1} \biggr)^{k+1}
	- \frac{h}{\mu(B_r)} \biggl( \frac{2e\log(\mu(B_r)/(2^kh))}{k} \biggr)^k.
	\end{align*}
	By choosing $k = \lfloor a \log(2p+1) \rfloor$ where $a \geq 1$ and $a \log(2e/a) < 1$, simple calculation shows that
	\begin{align*}
	& \frac{1}{2p+1} \biggl( \frac{2e\log(2p+1)}{k+1} \biggr)^{k+1}
	+ \frac{h}{\mu(B_r)} \biggl( \frac{2e\log(\mu(B_r)/(2^kh))}{k} \biggr)^k
	\\
	& \leq \frac{1}{2p+1} \biggl( \frac{2e\log(2p+1)}{a\log(2p+1)+1} \biggr)^{a\log(2p+1)+1}
	+ \frac{h}{\mu(B_r)} \biggl( \frac{2e\log(\mu(B_r)/(2^{a\log(2p+1)}h))}{a\log(2p+\mathrm{l})} \biggr)^{a\log(2p+1)} 
	\\
	& \leq p^{a\log(2e/a)-1} + p^a \biggl( \frac{h}{\mu(B_r)} \biggr)^{1/2}
	\\
	& \leq 2 p^a \biggl( \frac{h}{\mu(B_r)} \biggr)^{1/2}
	\end{align*}
	if $a \geq 4.33$. By taking $e^{-\tau} := 2 p^a (h/\mu(B_r))^{1/2}$, we obtain that
	\begin{align} \label{ProbMuAGG}
	\mathrm{P}_Z \Bigl( \mu(A) \geq (1/(2e^{\tau}p^a))^2 \mu(B_r) \Bigr) \geq 1 - e^{-\tau}.
	\end{align}
	Finally, combining \eqref{EstimErrorInfty2} with \eqref{ProbMuAGG} where we take $h = (1/(2e^{\tau} p^a))^2 \mu(B_r)$, we obtain that with probability $\mathrm{P}^n \otimes \mathrm{P}_Z$ at least $1 - 2 e^{-\tau}$,
	\begin{align}
	\|f_{\mathrm{D},Z} - f_{\mathrm{P},Z}\|_{\infty}
	& \leq \sqrt{\frac{8 \|f\|_{\infty} p^{2a} ((8d+1)\tau + 23\log n + 8ad\log p)}{n \mu(B_r) e^{-2\tau}}}
	\\
	& \phantom{=}
	+ \frac{8 p^{2a} ((8d+1)\tau + 23 \log n + 8ad \log p)}{3 n \mu(B_r) e^{-2\tau}} + \frac{2}{n}.  
	\label{EstimErrorInfty3}
	\end{align}
\end{proof}

\begin{proof}[of Theorem \ref{L1Convergence}] 
	From Proposition \ref{ApproximationError}, it can be seen that when $p_n \to \infty$, for any $\varepsilon > 0$, there exists a constant $n_1$ such that for $n \geq n_1$, we have $\|f_{\mathrm{P},Z} - f\|_1 \leq \varepsilon$ with probability $\mathrm{P}_Z$ at least $1 - \xi$ and $r_n := p_n^{c_T/(8d)} \to \infty$.
	
	For any $0 < \xi < 1$, we select $\tau := \log(1/\xi)$. As a result, there exists a constant $n_2$ such that $\tau < p_n \log n$ for any $n \geq n_2$. Since $p_n \log n / n \to 0$, following from Proposition \ref{OracleInequalityL1}, there also exists a constant $n_3$ such that for all $n \geq n_3$, we have $\|f_{\mathrm{D},Z} - f_{\mathrm{P},Z}\|_1 \leq \varepsilon$.
	
	Therefore, when $n \geq \max \{ n_1, n_2, n_3 \}$ is sufficient large, for any $\varepsilon > 0$, with probability $\mathrm{P}^n$ at least $1 - 2 \xi$, there holds
	\begin{align*}
	\|f_{\mathrm{D},Z} - f\|_1 \leq 2 \varepsilon.
	\end{align*}
	Therefore, with properly chosen $\xi$, one can show that $f_{\mathrm{D},Z}$ converges to $f$ under $L_1$-norm almost surely. We have completed the proof of Theorem \ref{L1Convergence}. 
\end{proof}

\begin{proof}[of Theorem \ref{L1ConvergenceRates}] 
	\emph{(i)} 
	Combining the estimates in Proposition \ref{ApproximationError} and Proposition \ref{L1LInftyRelation} and Proposition \ref{OracleInequalityL1}, we know that with probability $\mathrm{P}^n \otimes \mathrm{P}_Z$ at least $1 - 2 e^{-\tau}$, there holds
	\begin{align*}
	\|f_{\mathrm{D},Z} - f\|_1
	& \leq \sqrt{\frac{9(\tau+15dp\log n)}{2n}} + \frac{2(\tau+15dp\log n)}{n} + \frac{4}{n}
	\\
	& \phantom{=}
	+ c 2^d r^d ( (K+2r) d e^{\tau} p^{-c_T/(4d)})^{\alpha} + 2 P(B_r^c)
	\\
	& \leq C \biggl( \sqrt{\frac{(\tau+dp\log n)}{n}} + r^d (re^{\tau}dp^{-c_T/(4d)})^{\alpha} + P(B_r^c) \biggr)
	\end{align*}
	where $C$ is some constant such that the last inequality holds. Therefore, for any $r \geq 1$, $p \in \mathbb{N}_+$, with probability $\mathrm{P}^n \otimes \mathrm{P}_Z$ at least $1 - 2 e^{-\tau}$ there holds
	\begin{align} \label{ApproxError1}
	\|f_{\mathrm{D},Z} - f\|_1
	\leq C \biggl( \sqrt{\frac{(\tau+dp\log n)}{n}} + r^d (re^{\tau}dp^{-c_T/(4d)})^{\alpha} + r^{-\eta d} \biggr).
	\end{align}
	By choosing
	\begin{align*}
	r_n = \biggl( \frac{(c_T\alpha+2d)\eta}{2(d+\alpha)C_{\alpha,\tau,d}} \biggr)^{\frac{c_T\alpha+2d}{2d(d+\alpha)+\eta d(c_T\alpha+2d)}}
	\biggl( \frac{n}{\log n} \biggr)^{\frac{c_T\alpha}{4d(d+\alpha)+2\eta d(c_T\alpha+2d)}}
	\end{align*}
	and
	\begin{align*}
	p_n = \biggl( \frac{c_T\alpha}{2} d^{\alpha-\frac{3}{2}} e^{\tau\alpha} \biggr)^{\frac{4d}{c_T\alpha+2d}}
	\biggl( \frac{n}{\log n} \biggr)^{\frac{2d}{c_T\alpha+2d}}
	r^{\frac{4d(d+\alpha)}{c_T\alpha+2d}},
		\end{align*}
		then we have
		\begin{align} \label{ApproxError2}
		\|f_{\mathrm{D},Z} - f\|_1
		\leq C_{\alpha,\tau,d,\eta} \biggl( \frac{\log n}{n} \biggr)^{\frac{c_T\alpha\eta}{2\alpha(c_T\eta+2)+4d(\eta+1)}},
		\end{align}
		where constant is
		\begin{align*}
		C_{\alpha,\tau,d,\eta}
		:= C \biggl( \frac{2C_{\alpha,\tau,d}(d+\alpha)}{(c_T\alpha+2d)\eta} \biggr)^{\frac{(2d+c_T\alpha)\eta}{2d(1+\eta)+\alpha(2+c_T\eta)}}
		\biggl( 1 + \frac{(c_Td+2d)\eta}{2(d+\alpha)} \biggr) 
		+ C \tau^{\frac{1}{2}}.
		\end{align*}
		
		\emph{(ii)} 
		Similar to case \emph{(i)}, one can show that with probability $\mathrm{P}^n \otimes \mathrm{P}_Z$ at least $1 - 2 e^{-\tau}$ there holds
		\begin{align*}
		\|f_{\mathrm{D},Z} - f\|_1
		\leq C \biggl( \sqrt{\frac{(\tau+dp\log n)}{n}} + r^d (re^{\tau}dp^{-c_T/(4d)})^{\alpha} + e^{-ar^{\eta}} \biggr).
		\end{align*}
		By choosing
		\begin{align*}
		r_n = (\log n/a)^{\frac{1}{\eta}}
		\end{align*}
		and
		\begin{align*}
		p_n = \biggl( \frac{c_T\alpha}{2} d^{\alpha-\frac{3}{2}} e^{\tau\alpha} \biggr)^{\frac{4d}{c_T\alpha+2d}}
		\biggl( \frac{n}{\log n} \biggr)^{\frac{2d}{c_T\alpha+2d}}
		(\log n)^{\frac{d+\alpha}{\eta} \frac{4d}{c_T\alpha+2d}},
		\end{align*}
		then we have
		\begin{align*}
		\|f_{\mathrm{D},Z} - f\|_1
		\leq C_{\alpha,\tau,d,a} \biggl( \frac{\log n}{n} \biggr)^{\frac{c_T\alpha}{2c_T\alpha+4d}}
		(\log n)^{\frac{2d}{\eta} \frac{\alpha+d}{c_T\alpha+2d}},
		\end{align*}
		where constant is
		\begin{align*}
		C_{\alpha,\tau,d,a} := C (C_{\alpha,\tau,d} + \tau^{\frac{1}{2}} + 1).
		\end{align*}
		
		\emph{(iii)} Once again, similar to case \emph{(i)}, it can be showed that with confidence $\mathrm{P}^n \otimes \mathrm{P}_Z$ at least $1 - 2 e^{-\tau}$ there holds
		\begin{align*}
		\|f_{\mathrm{D},Z} - f\|_1
		\leq C \biggl( \sqrt{\frac{(\tau+dp\log n)}{n}} + r_0^d (r_0 e^{\tau}dp^{-c_T/(4d)})^{\alpha} \biggr).
		\end{align*}
		By choosing
		\begin{align*}
		p_n = \biggl( \frac{c_T\alpha}{2} d^{\alpha-\frac{3}{2}} e^{\tau\alpha} r_0^{d+\alpha} \biggr)^{\frac{4d}{c_T\alpha+2d}}
		\biggl( \frac{n}{\log n} \biggr)^{\frac{2d}{c_T\alpha+2d}},
		\end{align*}
		then we have
		\begin{align*}
		\|f_{\mathrm{D},Z} - f\|_1
		\leq C_{\alpha,\tau,d,r_0} \biggl( \frac{\log n}{n} \biggr)^{\frac{c_T\alpha}{2c_T\alpha+2d}},
		\end{align*}
		where constant is
		\begin{align*}
		C_{\alpha,\tau,d,r_0} := C (C_{\alpha,\tau,d} r_0^{\frac{2d(d+\alpha)}{c_T\alpha+2d}} + \tau^{\frac{1}{2}}).
		\end{align*}
	\end{proof}

	\begin{proof}[of Theorem \ref{LInftyConvergenceRates}] 
		The desired estimate is an easy consequence if we combine the estimates in Proposition \ref{OracleInequalityLInfty} and Proposition \ref{ApproximationError} \emph{(ii)} and choose
		\begin{align*}
		p_n = (n/\log n)^{\frac{2d}{c_T\alpha+4ad}}.
		\end{align*}
		We omit the details of the proof here.
	\end{proof}

\begin{proof}[of Theorem \ref{L1ConvergenceRatesForest}] 
	Considering the relationship between the $L_1$ error of the best-scored random density forest estimator and the $L_1$ errors of individual best-scored random density tree estimators contained in the forest, there holds
	\begin{align*}
	\|f_{\mathrm{D},Z_{\mathbb{E}}} - f\|_1
	= \biggl\| \frac{1}{m} \sum_{t=1}^m f_{\mathrm{D},Z_t} - f \biggr\|_1
	= \biggl\| \frac{1}{m} \sum_{t=1}^m (f_{\mathrm{D},Z_t} - f) \biggr\|_1
	\leq \frac{1}{m} \sum_{t=1}^m \|f_{\mathrm{D},Z_t} - f\|_1.
	\end{align*}
	In the following, we only present the analysis of the case \emph{(i)} in the three tail probability distributions, since the proof of \emph{(ii)} and \emph{(iii)} are quite similar.
	
	According to Theorem \ref{L1ConvergenceRates}, with $p_n = (n/\log n)^{\frac{2d(\eta+1)+2\alpha}{\alpha(c_T\eta+2)+2d(\eta+1)}}$, the bound of the $L_1$ errors of individual estimators contained in the ensemble is denoted by
	\begin{align*}
	\mathcal{E} := C_{\alpha,\tau,d,\eta} \biggl( \frac{\log n}{n} \biggr)^{\frac{c_T\alpha\eta}{2\alpha(c_T\eta+2)+4d(\eta+1)}}
	\end{align*}
	as in \eqref{ApproxError1}. Then the union bound yields that, for all $\tau > 0$, there holds
	\begin{align*}
	\mathrm{P}^n \otimes \mathrm{P}_Z ( \|f_{\mathrm{D},Z_{\mathbb{E}}} - f\|_1 > \mathcal{E} )
	\leq \sum_{t=1}^m \mathrm{P}^n \otimes \mathrm{P}_Z ( \|f_{\mathrm{D},Z_t} - f\|_1 > \mathcal{E} )
	\leq 2 m e^{-\tau}.
	\end{align*}
	Consequently, with probability $\mathrm{P}^n \otimes \mathrm{P}_Z$ at least $1 - 2 e^{-\tau}$, there holds
	\begin{align*}
	\|f_{\mathrm{D},Z_{\mathbb{E}}} - f\|_1
	\leq C_{\alpha,\tau,d,\eta,m} \biggl( \frac{\log n}{n} \biggr)^{\frac{c_T\alpha\eta}{2\alpha(c_T\eta+2)+4d(\eta+1)}},
	\end{align*}
	where
	\begin{align*}
	C_{\alpha,\tau,d,m} := \biggl( \frac{c_T\alpha m^{\alpha}e^{\tau\alpha}}{2} \biggr)^{\frac{2d}{c_T\alpha+2d}}
	\biggl( 1 + \frac{2d}{c_T\alpha} \biggr) 
	d^{\frac{4\alpha d+c_T\alpha-4d}{2c_T\alpha+4d}}
	\end{align*}
	and
	\begin{align*}
	C_{\alpha,\tau,d,\eta,m} := C \biggl( \frac{2C_{\alpha,\tau,d,m}(d+\alpha)}{(c_T\alpha+2d)\eta} \biggr)^{\frac{(2d+c_T\alpha)\eta}{2d(1+\eta)+\alpha(2+c_T\eta)}}
	\biggl( 1 + \frac{(c_Td+2d)\eta}{2(d+\alpha)} \biggr)
	+ C (\tau + \log m)^{\frac{1}{2}}.
	\end{align*}
\end{proof}

\begin{proof}[of Theorem \ref{LInftyConvergenceRatesForest}] 
	We consider the relationship between the $L_{\infty}$ error of the best-scored random density forest estimator and the $L_{\infty}$ errors of individual estimators contained in the forest.
	\begin{align*}
	\|f_{\mathrm{D},Z_{\mathbb{E}}} - f\|_{\infty}
	= \biggl\| \frac{1}{m} \sum_{t=1}^m f_{\mathrm{D},Z_t} - f \biggr\|_{\infty}
	= \biggl\| \frac{1}{m} \sum_{t=1}^m (f_{\mathrm{D},Z_t} - f) \biggr\|_{\infty}
	\leq \frac{1}{m} \sum_{t=1}^m \|f_{\mathrm{D},Z_t} - f\|_{\infty}.
	\end{align*}
	The desired estimate can be obtained by choosing the same number of splits for each partition in $\{ \mathcal{A}_{Z_t} \}_{t=1}^m$, which is
	\begin{align*}
	p_n = (n/\log n)^{\frac{2d}{c_T\alpha+4ad}}.
	\end{align*}
	We omit the details of the proof here for it is similar to that of the Theorem \ref{L1ConvergenceRatesForest}. 
\end{proof}

\begin{proof}[of Lemma \ref{VCindex}] 
	The proof will be conducted by dint of geometric constructions, and we proceed by induction. We begin by observing a partition with number of splits $p = 1$. On account that the dimension of the feature space is $d$, the smallest number of points that cannot be divided by $p = 1$ split is $d + 2$. Specifically, considering the fact that $d$ points can be used to form $d - 1$ independent vectors and therefore a hyperplane of a $d$-dimensional space, we now focus on the case where there is a hyperplane consisting of $d$ points all from the same class labeled as $A$, and there are two points from the other class $B$ on either side of the hyperplane. We denote the hyperplane by $H_1^A$ for brevity. In this case, points from two classes cannot be separated by one split, i.e. one hyperplane, which means that $\mathrm{VC}(\mathcal{B}(\pi_1)) \leq d + 2$.
	
	We next turn to consider the partition with number of splits $p = 2$ which is an extension of the above case. Once we pick one point out of the two located on either side of the above hyperplane $H_1^A$, a new hyperplane $H_2^B$ parallel to $H_1^A$ can be constructed by combining the selected point with $d - 1$ newly-added points from class $B$. Subsequently, a new point from class $A$ is added to the side of the newly constructed hyperplane $H_2^B$. Notice that the newly added point should be located on the opposite side to $H_1^A$. Under this situation, $p = 2$ splits can never separate those $2 d + 2$ points from two different classes. As a result, we prove that $\mathrm{VC}(\mathcal{B}(\pi_2)) \leq 2 d + 2$.
	
	If we apply induction to the above cases, the analysis of VC index can be extended to the general case where $p \in \mathbb{N}$. What we need to do is to add new points continuously to form $p$ mutually parallel hyperplanes with any two adjacent hyperplanes being built from different classes. Without loss of generality, we assume that $p = 2k+1$, $k \in \mathbb{N}$, and there are two points denoted by $p_1^B, p_2^B$ from class $B$ separated by $2 k + 1$ alternately appearing hyperplanes. Their locations can be represented by $p_1^B, H_1^A, H_2^B, H_3^A, H_4^B, \ldots, H_{(2k+1)}^A, p_2^B$. According to this construction, we demonstrate that the smallest number of points that cannot be divided by $p$ splits is $d p + 2$, which leads to $\mathrm{VC}(\mathcal{B}(\pi_p)) \leq d p + 2$.
	
	It should be noted that our hyperplanes can be generated both vertically and obliquely, which is in line with our splitting criteria for the random trees. This completes the proof. 
\end{proof}

\begin{proof}[of Lemma \ref{ScriptBpCoveringNumber}]
	The inequality \eqref{BpCoveringNumber} follows directly from Lemma \ref{VCindex} and Theorem 9.2 in \cite{Kosorok2008introduction}.
\end{proof}

\section{Conclusion}\label{sec::conclusion}

In the present paper, we study the density estimation problems with a new nonparametric strategy called the best-scored random forest density estimation. Each tree in the forest is the best one selected from a certain number of purely random density tree candidates based on their estimation performance and we call this selection mechanism the best-scored method. It is by this best-scored method, the ensemble forest of these selected trees are able to achieve better estimation results. The main results presented in this paper include establishing the consistency of the best-scored random density tree estimators under $L_1$-norm, and the convergence rates of them under $L_1$-norm with respect to three different tail assumptions. The convergence rates of the best-scored random density tree estimators are also analyzed under $L_{\infty}$-norm with mild assumptions that the density function is $\alpha$-H\"{o}lder continuous, bounded and compactly supported. Moreover, we also establish the above convergence rates analysis for the best-scored random density forest estimators. When it comes to the numerical experiments, in addition to improving the purely random splitting criteria to data-driven ones called the adaptive random splitting criteria, we also generalize the axis-parallel partitions to the oblique ones which are proved to increase the estimation accuracy. Experiments on comparisons between our forest density estimators and other standard density estimators further illustrate the satisfying performance of our proposal with respect to the estimation accuracy and stronger resistance to the curse of dimensionality.

\acks{The authors are grateful to Professor Ingo Steinwart for his valuable comments and suggestions. The research leading to these results are supported by fund for building world-class universities (disciplines) of Renmin University of China. The corresponding author is Hongwei Wen.}


\begin{thebibliography}{42}
	\providecommand{\natexlab}[1]{#1}
	\providecommand{\url}[1]{\texttt{#1}}
	\expandafter\ifx\csname urlstyle\endcsname\relax
	\providecommand{\doi}[1]{doi: #1}\else
	\providecommand{\doi}{doi: \begingroup \urlstyle{rm}\Url}\fi
	
	\bibitem[Botev et~al.(2010)Botev, Grotowski, and Kroese]{botev2010kernel}
	Z.~I. Botev, J.~F. Grotowski, and D.~P. Kroese.
	\newblock Kernel density estimation via diffusion.
	\newblock \emph{The Annals of Statistics}, 38\penalty0 (5):\penalty0
	2916--2957, 2010.
	
	\bibitem[Bremain(2000)]{bremain2000some}
	Leo Bremain.
	\newblock Some infinite theory for predictor ensembles.
	\newblock \emph{University of California at Berkeley Papers}, 2000.
	
	\bibitem[Chang(2016)]{chang2016nonparametric}
	Ju~Yong Chang.
	\newblock Nonparametric feature matching based conditional random fields for
	gesture recognition from multi-modal video.
	\newblock \emph{IEEE transactions on pattern analysis and machine
		intelligence}, 38\penalty0 (8):\penalty0 1612--1625, 2016.
	
	\bibitem[Devroye(1986)]{devroye1986a}
	Luc Devroye.
	\newblock A note on the height of binary search trees.
	\newblock \emph{Journal of the Association for Computing Machinery},
	33\penalty0 (3):\penalty0 489--498, 1986.
	
	\bibitem[Devroye and Gy\"{o}rfi(1983)]{devroye1983distributionfree}
	Luc Devroye and L{\'a}szl{\'o} Gy\"{o}rfi.
	\newblock Distribution-free exponential bound on the l1 error of partitioning
	estimates of a regression function.
	\newblock In \emph{Proceedings of the Fourth Pannonian Symposium on
		Mathematical Statistics}, pages 67--76. Akad{\'e}miai Kiad{\'o} Budapest,
	Hungary, 1983.
	
	\bibitem[Devroye and Gy\"{o}rfi(1985)]{devroye1985nonparametric}
	Luc Devroye and L{\'a}szl{\'o} Gy\"{o}rfi.
	\newblock \emph{Nonparametric Density Estimation: The $L_1$ View}, volume 119.
	\newblock John Wiley \& Sons Incorporated, 1985.
	
	\bibitem[Doucet et~al.(2001)Doucet, De~Freitas, and Gordon]{doucet2001an}
	Arnaud Doucet, Nando De~Freitas, and Neil Gordon.
	\newblock An introduction to sequential monte carlo methods.
	\newblock In \emph{Sequential Monte Carlo methods in practice}, pages 3--14.
	Springer, 2001.
	
	\bibitem[Doukhan and León(1990)]{doukhan1990deviation}
	Paul Doukhan and Jos~R León.
	\newblock D{\'e}viation quadratique d'estimateurs de densit{\'e} par
	projections orthogonales.
	\newblock \emph{Comptes Rendus de l'Acad{\'e}mie des Sciences-Series
		I-Mathematics}, 310:\penalty0 425–430, 1990.
	
	\bibitem[Escobar and West(1995)]{escobar1995bayesian}
	Michael~D. Escobar and Mike West.
	\newblock Bayesian density estimation and inference using mixtures.
	\newblock \emph{Journal of the American Statistical Association}, 90\penalty0
	(430):\penalty0 577--588, 1995.
	
	\bibitem[Fraley and Raftery(2002)]{fraley2002modelbased}
	Chris Fraley and Adrian~E. Raftery.
	\newblock Model-based clustering, discriminant analysis, and density
	estimation.
	\newblock \emph{Journal of the American Statistical Association}, 97\penalty0
	(458):\penalty0 611--631, 2002.
	
	\bibitem[Freedman and Diaconis(1981)]{freedman1981on}
	David Freedman and Persi Diaconis.
	\newblock On the histogram as a density estimator: $l_2$ theory.
	\newblock \emph{Probability Theory and Related Fields}, 57\penalty0
	(4):\penalty0 453–476, 1981.
	
	\bibitem[Ghosal(2001)]{ghosal2001convergence}
	Subhashis Ghosal.
	\newblock Convergence rates for density estimation with {B}ernstein
	polynomials.
	\newblock \emph{The Annals of Statistics}, 29\penalty0 (5):\penalty0
	1264--1280, 2001.
	
	\bibitem[Glick(1973)]{glick1973samplebased}
	Ned Glick.
	\newblock Sample-based multinomial classification.
	\newblock \emph{Biometrics. Journal of the Biometric Society}, 29:\penalty0
	241--256, 1973.
	
	\bibitem[Gordon and Olshen(1978)]{gordon1978asymptotically}
	Louis Gordon and Richard~A. Olshen.
	\newblock Asymptotically efficient solutions to the classification problem.
	\newblock \emph{The Annals of Statistics}, 6\penalty0 (3):\penalty0 515--533,
	1978.
	
	\bibitem[Gordon and Olshen(1980)]{gordon1980consistent}
	Louis Gordon and Richard~A. Olshen.
	\newblock Consistent nonparametric regression from recursive partitioning
	schemes.
	\newblock \emph{Journal of Multivariate Analysis}, 10\penalty0 (4):\penalty0
	611--627, 1980.
	
	\bibitem[Hang et~al.(2018)Hang, Steinwart, Feng, and Suykens]{hang2018kernel}
	Hanyuan Hang, Ingo Steinwart, Yunlong Feng, and Johan A.~K. Suykens.
	\newblock Kernel density estimation for dynamical systems.
	\newblock \emph{Journal of Machine Learning Research}, 19:\penalty0 1–49,
	2018.
	
	\bibitem[H\"{a}rdle et~al.(2012)H\"{a}rdle, M\"{u}ller, Sperlich, and
	Werwatz]{hardle2012nonparametric}
	Wolfgang~Karl H\"{a}rdle, Marlene M\"{u}ller, Stefan Sperlich, and Axel
	Werwatz.
	\newblock \emph{Nonparametric and semiparametric models}.
	\newblock Springer Science \& Business Media, 2012.
	
	\bibitem[Hwang et~al.(1994)Hwang, Lay, and Lippman]{hwang1994nonparametric}
	Jenq-Neng Hwang, Shyh-Rong Lay, and Alan Lippman.
	\newblock Nonparametric multivariate density estimation: a comparative study.
	\newblock \emph{IEEE Transactions on Signal Processing}, 42\penalty0
	(10):\penalty0 2795--2810, 1994.
	
	\bibitem[Ihsani and Farncombe(2016)]{ihsani2016a}
	Alvin Ihsani and Troy~H Farncombe.
	\newblock A kernel density estimator-based maximum a posteriori image
	reconstruction method for dynamic emission tomography imaging.
	\newblock \emph{IEEE Transactions on Image Processing}, 25\penalty0
	(5):\penalty0 2233–2248, 2016.
	
	\bibitem[Jiang(2017)]{jiang2017the}
	Heinrich Jiang.
	\newblock The uniform convergence rate of kernel density estimate.
	\newblock In \emph{International Conference on Machine Learning}, page
	1694–1703, 2017.
	
	\bibitem[Jiang et~al.(2017)Jiang, Sun, Zhao, and Yang]{jiang2017fog}
	Yutong Jiang, Changming Sun, Yu~Zhao, and Li~Yang.
	\newblock Fog density estimation and image defogging based on surrogate
	modeling for optical depth.
	\newblock \emph{IEEE Transactions on Image Processing}, 26\penalty0
	(7):\penalty0 3397--3409, 2017.
	
	\bibitem[Kerkyacharian and Picard(1992)]{kerkyacharian1992density}
	G{\'e}rard Kerkyacharian and Dominique Picard.
	\newblock Density estimation in {B}esov spaces.
	\newblock \emph{Statistics \& Probability Letters}, 13\penalty0 (1):\penalty0
	15--24, 1992.
	
	\bibitem[Klemel\"{a}(2009)]{klemela2009multivariate}
	Jussi Klemel\"{a}.
	\newblock Multivariate histograms with data-dependent partitions.
	\newblock \emph{Statistica Sinica}, 19\penalty0 (1):\penalty0 159--176, 2009.
	
	\bibitem[Kosorok(2008)]{Kosorok2008introduction}
	Michael~R. Kosorok.
	\newblock \emph{Introduction to empirical processes and semiparametric
		inference}.
	\newblock Springer Series in Statistics. Springer, New York, 2008.
	
	\bibitem[Liu et~al.(2016)Liu, Wang, Feng, and Xi]{liu2016highway}
	Xu~Liu, Zilei Wang, Jiashi Feng, and Hongsheng Xi.
	\newblock Highway vehicle counting in compressed domain.
	\newblock In \emph{Proceedings of the IEEE Conference on Computer Vision and
		Pattern Recognition}, pages 3016--3024, 2016.
	
	\bibitem[L{\'o}pez-Rubio(2014)]{lopezrubio2014a}
	Ezequiel L{\'o}pez-Rubio.
	\newblock A histogram transform for probability density function estimation.
	\newblock \emph{IEEE Transactions on Pattern Analysis and Machine
		Intelligence}, 36\penalty0 (4):\penalty0 644--56, 2014.
	
	\bibitem[Lugosi and Nobel(1996)]{lugosi1996consistency}
	G\'{a}bor Lugosi and Andrew Nobel.
	\newblock Consistency of data-driven histogram methods for density estimation
	and classification.
	\newblock \emph{The Annals of Statistics}, 24\penalty0 (2):\penalty0 687--706,
	1996.
	
	\bibitem[Ma et~al.(2015)Ma, Yu, and Chan]{ma2015small}
	Zheng Ma, Lei Yu, and Antoni~B. Chan.
	\newblock Small instance detection by integer programming on object density
	maps.
	\newblock In \emph{Proceedings of the IEEE Conference on Computer Vision and
		Pattern Recognition}, pages 3689--3697, 2015.
	
	\bibitem[Parzen(1962)]{parzen1962on}
	Emanuel Parzen.
	\newblock On estimation of a probability density function and mode.
	\newblock \emph{Annals of Mathematical Statistics}, 33:\penalty0 1065--1076,
	1962.
	
	\bibitem[Roeder and Wasserman(1997)]{roeder1997practical}
	Kathryn Roeder and Larry Wasserman.
	\newblock Practical {B}ayesian density estimation using mixtures of normals.
	\newblock \emph{Journal of the American Statistical Association}, 92\penalty0
	(439):\penalty0 894--902, 1997.
	
	\bibitem[Rosenblatt(1956)]{rosenblatt1956remarks}
	Murray Rosenblatt.
	\newblock Remarks on some nonparametric estimates of a density function.
	\newblock \emph{Annals of Mathematical Statistics}, 27:\penalty0 832--837,
	1956.
	
	\bibitem[Scott(2015)]{scott1992multivariate}
	David~W. Scott.
	\newblock \emph{Multivariate density estimation}.
	\newblock Wiley Series in Probability and Statistics. John Wiley \& Sons, Inc.,
	Hoboken, NJ, second edition, 2015.
	\newblock Theory, practice, and visualization.
	
	\bibitem[Silverman(1986)]{silverman1986density}
	B.~W. Silverman.
	\newblock \emph{Density estimation for statistics and data analysis}.
	\newblock Monographs on Statistics and Applied Probability. Chapman \& Hall,
	London, 1986.
	
	\bibitem[Simonoff(1996)]{simonoff1996smoothing}
	Jeffrey~S. Simonoff.
	\newblock \emph{Smoothing methods in statistics}.
	\newblock Springer Series in Statistics. Springer-Verlag, New York, 1996.
	
	\bibitem[Tang et~al.(2012)Tang, Salakhutdinov, and Hinton]{tang2012deep}
	Yichuan Tang, Ruslan Salakhutdinov, and Geoffrey Hinton.
	\newblock Deep mixtures of factor analysers.
	\newblock In \emph{Proceedings of the 29th International Conference on Machine
		Learning}, pages 505--512, 2012.
	
	\bibitem[Terrell and Scott(1992)]{terrell1992variable}
	George~R. Terrell and David~W. Scott.
	\newblock Variable kernel density estimation.
	\newblock \emph{The Annals of Statistics}, 20\penalty0 (3):\penalty0
	1236--1265, 1992.
	
	\bibitem[Uria et~al.(2013)Uria, Murray, and Larochelle]{uria2013rnade}
	Benigno Uria, Iain Murray, and Hugo Larochelle.
	\newblock Rnade: the real-valued neural autoregressive density-estimator.
	\newblock In \emph{International Conference on Neural Information Processing
		Systems}, page 2175–2183, 2013.
	
	\bibitem[Uria et~al.(2016)Uria, C\^{o}t\'{e}, Gregor, Murray, and
	Larochelle]{uria2016neural}
	Benigno Uria, Marc-Alexandre C\^{o}t\'{e}, Karol Gregor, Iain Murray, and Hugo
	Larochelle.
	\newblock Neural autoregressive distribution estimation.
	\newblock \emph{Journal of Machine Learning Research}, 17, 2016.
	
	\bibitem[van~der Vaart and Wellner(1996)]{vandervaart1996weak}
	Aad~W. van~der Vaart and Jon~A. Wellner.
	\newblock \emph{Weak convergence and empirical processes}.
	\newblock Springer Series in Statistics. Springer-Verlag, New York, 1996.
	\newblock With applications to statistics.
	
	\bibitem[Vestner et~al.(2017)Vestner, Litman, Rodol\'{a}, Bronstein, and
	Cremers]{vestner2017product}
	Matthias Vestner, Roee Litman, Emanuele Rodol\'{a}, Alexander~M Bronstein, and
	Daniel Cremers.
	\newblock Product manifold filter: Non-rigid shape correspondence via kernel
	density estimation in the product space.
	\newblock In \emph{Proceedings of the IEEE Conference on Computer Vision and
		Pattern Recognition}, page 6681–6690, 2017.
	
	\bibitem[Wang et~al.(2018)Wang, Zou, and Wang]{wang2018manifoldbased}
	Yi~Wang, Yuexian Zou, and Wenwu Wang.
	\newblock Manifold-based visual object counting.
	\newblock \emph{IEEE Transactions on Image Processing}, 27\penalty0
	(7):\penalty0 3248--3263, 2018.
	
	\bibitem[Zhou et~al.(2018)Zhou, Rangarajan, and Gader]{zhou2018a}
	Yuan Zhou, Anand Rangarajan, and Paul~D. Gader.
	\newblock A {G}aussian mixture model representation of endmember variability in
	hyperspectral unmixing.
	\newblock \emph{IEEE Transactions on Image Processing}, 27\penalty0
	(5):\penalty0 2242--2256, 2018.
	
\end{thebibliography}

\end{document}